\documentclass[11pt,a4paper]{article}
\usepackage{geometry}
\usepackage{graphicx,epsfig}
\usepackage{amsfonts}
\usepackage{amsthm, amsmath}
\theoremstyle{definition}
\theoremstyle{plain}
\newtheorem{definition}{Definition}
\newtheorem{theorem}{Theorem}
\usepackage{float}
\usepackage{cite}
\usepackage[hyphens]{url}
\usepackage[hidelinks]{hyperref}
\usepackage{amsbsy}
\usepackage{algorithm}
\usepackage{algpseudocode}
\usepackage{amssymb}
\usepackage{tabu}
\providecommand{\keywords}[1]{\small\textbf{\textit{Keywords---}}#1}
\usepackage{authblk}

\title{Kernel Alignment for Unsupervised Feature Selection via Matrix Factorization}
\author[ ]{Ziyuan Lin}
\author[1]{Deanna Needell}
\affil[1]{University of California, Los Angeles, Department of Mathematics}
\affil[ ]{e-mail: ziyuan.1.lin@gmail, deanna@math.ucla.edu}
\date{March 13, 2024}

\begin{document}

\maketitle

\begin{abstract}
By removing irrelevant and redundant features, feature selection aims to find a good representation of the original features. With the prevalence of unlabeled data, unsupervised feature selection has been proven effective in alleviating the so-called curse of dimensionality. Most existing matrix factorization-based unsupervised feature selection methods are built upon subspace learning, but they have limitations in capturing nonlinear structural information among features. It is well-known that kernel techniques can capture nonlinear structural information. In this paper, we construct a model by integrating kernel functions and kernel alignment, which can be equivalently characterized as a matrix factorization problem. However, such an extension raises another issue: the algorithm performance heavily depends on the choice of kernel, which is often unknown a priori. Therefore, we further propose a multiple kernel-based learning method. By doing so, our model can learn both linear and nonlinear similarity information and automatically generate the most appropriate kernel. Experimental analysis on real-world data demonstrates that the two proposed methods outperform other classic and state-of-the-art unsupervised feature selection methods in terms of clustering results and redundancy reduction in almost all datasets tested.
\end{abstract}
\keywords{Unsupervised feature selection, Kernel alignment, Matrix factorization, Multiple kernel learning}

\section{Introduction}

High-dimensional data appears in multiple scientific fields like computer vision \cite{parvaiz2023vision, roy2023wildect}, bioinformatics \cite{chen2023graph} and social networks \cite{sharma2023hybrid}. Dealing with such high-dimensional data significantly amplifies the complexity of data processing and computation, demanding increased storage space and computation time. Moreover, data with many dimensions often encompasses irrelevant, redundant, or even noisy features, which can adversely impact the performance of various learning algorithms, affecting outcomes like clustering and classification \cite{parsons2004subspace, fan2008high}. As a result, the pursuit of approximating the original data using fewer but high-quality features has emerged as a critical research focus, and dimensionality reduction algorithms provide a viable solution to address this challenge.

In dimensionality reduction techniques, feature selection \cite{guyon2003introduction} and feature extraction \cite{guyon2008feature} stand as the two primary methods. The former involves selecting a small subset of original features based on certain criteria, while the latter seeks a projection to map the original high-dimensional data into a lower-dimensional subspace. In other words, compared to feature extraction, feature selection algorithms retain the original features of the dataset, enhancing the interpretability of the model's outcomes. Moreover, in numerous data mining applications, sample labels remain unknown, rendering supervised and semi-supervised dimensionality reduction algorithms ineffective. In contrast, unsupervised learning methods solely rely on the intrinsic properties of data without necessitating class labels, thus making the development of unsupervised feature selection (UFS) algorithms indispensable \cite{solorio2020review}.

Subspace learning is regarded as a successful approach in the realm of unsupervised feature selection \cite{zhou2016global, wang2020discriminative, parsa2020unsupervised}. It aims to discover a mapping that can effectively project the original high-dimensional space into a lower-dimensional subspace that is representative and relevant. Subspace learning allows the use of powerful tools such as matrix factorization and various regularization frameworks. These tools not only can compress the feature space into smaller dimensions, but also can efficiently remove noise and redundant features from the data \cite{nie2020subspace}.

Within subspace learning, the concept of subspace distance plays a critical role. As one of the evaluation criteria for feature selection, subspace distance directly influences the learned mapping of the algorithm, subsequently impacting the quality of the selected features. For instance, the Matrix Factorization Feature Selection (MFFS) algorithm proposed in \cite{wang2015subspace} introduces the concept of directional distance, measuring the distance between the original data and the selected feature subset projection back to the original space. It allows non-negative matrix factorization tools to solve the feature selection model. The Regularized Self-representation (RSR) method \cite{zhu2015unsupervised} proposes self-representation distance, assuming that each feature of the original data can be well approximated by a linear combination of related features. This approach also discerns redundant features and mitigates interference from outliers using the \(l_{2,1}\)-norm. Furthermore, the Variance–Covariance Subspace Distance Feature Selection (VCSDFS) method \cite{karami2023unsupervised} introduces a subspace distance termed variance and covariance distance. It thoroughly considers the variance of individual features and their correlations, thereby selecting informative features. In addition, in subspace learning, regularization frameworks also play a crucial role in selecting effective features. For example, the feature  selection based on Maximum Projection and Minimum Redundancy (MPMR) method \cite{wang2015unsupervised} and the Regularized Matrix Factorization Feature Selection (RMFFS) method \cite{qi2018unsupervised}, both built upon the MFFS algorithm, respectively introduce a regularization framework based on linear correlation among a set of selected features and inner product regularization to reduce redundancy in the selected feature subset. The RSR method employs \(l_{2,1}\)-norm for sparse regularization of the feature selection matrix. Despite the good performance of these UFS methods, a crucial drawback is their exclusive focus on linear correlation between features. They do not capture the nonlinear relationships between features, leading to the inability to select more discriminative features.

In this paper, we introduces a kernel alignment-based subspace distance as an evaluation metric for subspace learning to construct an UFS model. This approach enables the consideration of nonlinear relationships between features to jointly assess the quality of each feature. Additionally, we employ non-negative matrix factorization techniques to generate efficient algorithms for model solution. Furthermore, a multiple kernel method for feature selection is also provided and its algorithm is developed.
\subsection{Other Related Work}

The notion of kernel alignment was first introduced by Cristianini et al. \cite{cristianini2001kernel}, which is used as a similarity measure between two kernel matrices. Cortes et al. \cite{cortes2012algorithms} improved upon kernel alignment by proposing unnormalized centered kernel alignment, which can cancel the effect of unbalanced class distribution.

In the context of UFS tasks, where guidance from class labels is lacking, for two samples in a given dataset, it is assumed that the perspectives on similarity between these two samples provided by two different kernels should be consistent. Therefore, maximizing the similarity between the two kernels, one computed from the original features and the other from the selected features, which can ensure that the selected features maximally preserve the information of the original features.

Several works have considered using kernel alignment in the design of UFS models. Wei et al. \cite{wei2016nonlinear} utilized unnormalized centered kernel alignment based on the Gaussian kernel to develop an UFS method. However, the derivation of their algorithm depends on the choice of kernel, making it challenging to generalize to other types of kernels. Karami et al. \cite{karami2023unsupervised} proposed an UFS method based on variance-covariance subspace distance. We consider that, from the perspective of kernel alignment, minimizing the variance-covariance subspace distance can be interpreted as maximizing the similarity between two linear kernels.  Xing et al. \cite{xing2021fairness}, based on kernel alignment, investigated the fairness issue of unsupervised feature selection by introducing protected attributes. The UFS method proposed by Palazzo et al. \cite{palazzo2020unsupervised} is time-consuming as it involves training two models. The approach employs autoencoders to construct a target kernel and then utilizes kernel alignment to select features. In our work, we focus on accelerating the solution of the kernel alignment based UFS model through matrix factorization. Additionally, taking into account the limitations of single kernel models, we propose an UFS model based on multiple kernel learning \cite{gonen2011multiple}.
\subsection{Contribution}

In this paper, we propose two novel UFS methods from the viewpoint of non-negative matrix factorization, which we term Kernel Alignment Unsupervised Feature Selection (KAUFS) and Multiple Kernel Alignment Unsupervised Feature Selection (MKAUFS). The main highlights of the proposed methods are summarized below:

\begin{itemize}
\item To capture the nonlinear relationships among features, we innovatively introduce kernel alignment into the modeling of UFS methods based on subspace learning.

\item Many existing UFS algorithms based on non-negative matrix factorization cannot handle datasets containing negative values. Our proposed algorithm can handle both nonnegative and negative values in the input data and the kernel matrix (Gram matrix), and we provide a proof of convergence for the algorithm.

\item Considering that the performance of a single kernel model heavily depends on the choice of the kernel, which is usually unknown and time-consuming to determine. Inspired by the multiple kernel learning approach \cite{gonen2011multiple}, we extend the KAUFS method to the MKAUFS method by constructing several candidate kernels and merging them to form a consensus kernel to alleviate this issue, and it also enables us to better exploit the heterogeneous features of real-world datasets.
\end{itemize}
\subsection{Organization}
The rest of this paper is structured as follows. The remainder of this section introduces the notation used in this paper. In Section \ref{The Framework of KAUFS}, we explain the details of the KAUFS method. In Section \ref{Algorithm and Convergence Analysis}, we derive an iterative update algorithm for the KAUFS method and prove its convergence. In Section \ref{Mutiple Kernel Method for KAUFS}, we extend KAUFS method to a multiple kernel version and design its algorithm. We also provide the computational complexity analysis for both algorithms. Numerical experiments are presented and analyzed in Section \ref{Numerical Experiments}. Finally, this paper is concluded in Section \ref{Conclusions}.
\subsection{Notation}
We denote scalars by lower or upper case letters. We denote matrices with boldface uppercase letters and vectors with boldface lowercase letters. \(\mathbf{I}_{k}\) is used to denote the identity matrix in \(\mathbb{R}^{k\times k}\). The matrix \(\mathbf{1}_{d\times k}\) is in \(\mathbb{R}^{d\times k}\) so that all entries are equal to one. An indicator matrix means that each element in the matrix is either 1 or 0, and any row or column can have at most one non-zero element. For a given matrix \(\mathbf{X}\), the Frobenius norm of \(\mathbf{X}\) is defined as \(\|\mathbf{X}\|_{F}\). \(\operatorname{Tr}\left(\cdot\right)\) is the trace operator. \(\mathbf{X}^{T}\) denotes the transpose of \(\mathbf{X}\). For two vectors \(\mathbf{a}\) and \(\mathbf{b}\), we denote their inner product as \(\left\langle \mathbf{a} ,\mathbf{b}\right\rangle\).

\section{The Framework of KAUFS}\label{The Framework of KAUFS}

In this section, we will comprehensively explain the detailed information required to describe the structure of the KAUFS method.

First, we are given a high-dimensional input data, which is represented by an $n\times d$ matrix \(\mathbf{X}=\left[\mathbf{x}_{1}; \mathbf{x}_{2}; \dots ;\mathbf{x}_{n}\right]=\left[\mathbf{f}_{1}, \mathbf{f}_{2}, \dots ,\mathbf{f}_{d}\right]\), where \(\mathbf{x}_{i}\) \(\in \mathbb{R}^{1\times d}\) is the \(i\)th sample, \(\mathbf{f}_{j}\) \(\in \mathbb{R}^{n\times 1}\) is the \(j\)th feature and \(n\) and \(d\) represent the number of samples and features, respectively. The goal of a feature selection framework is to choose a small subset of features from all features that capture the most valuable information, thus approximating and representing all features effectively. More specifically, given a fixed number \(k\), where \(k<d\), the objective of the KAUFS method is to select a set of \(k\) features \(\left \{ \mathbf{f}_{j_{1}}, \mathbf{f}_{j_{2}}, \dots ,\mathbf{f}_{j_{k}} \right \}\) from the dataset \(\mathbf{X}\). These selected features will form the selected sub-matrix \(\mathbf{X}_{I}=\left [ \mathbf{f}_{j_{1}}, \mathbf{f}_{j_{2}}, \dots ,\mathbf{f}_{j_{k}} \right ]\), such that \(\mathbf{X}_{I}\) can effectively approximate the original feature matrix \(\mathbf{X}=\left[\mathbf{f}_{1}, \mathbf{f}_{2}, \dots ,\mathbf{f}_{d}\right]\) according to some metric. Moreover, we can express \(\mathbf{X}_{I}\) as
\begin{equation}
\mathbf{X}_{I} = \mathbf{X}\mathbf{W}.\nonumber
\end{equation}
where \(\mathbf{W} \in \mathbb{R}^{d\times k}\) is referred to as the feature weight matrix and is an indicator matrix. 

After selecting the feature subset, leveraging the self-representation property, which assumes that linear combinations of related features can effectively approximate the features of the original data, we can further represent the \(k\) features in \(\mathbf{X}_{I}\) as a linear combination to approximate the \(d\) features in the original dataset \(\mathbf{X}\), expressed as \(\mathbf{X}\approx \mathbf{X}_{I}\mathbf{H}\). More specifically, 
\begin{equation}
\left\{\begin{array}{l}
\mathbf{f}_{1} \approx  \mathbf{X W h}_{1}=\mathbf{H}_{11} \mathbf{f}_{j_{1}}+\mathbf{H}_{21} \mathbf{f}_{j_{2}}+\cdots+\mathbf{H}_{k 1} \mathbf{f}_{j_{k}}, \\
\mathbf{f}_{2} \approx \mathbf{X W h}_{2}=\mathbf{H}_{12} \mathbf{f}_{j_{1}}+\mathbf{H}_{22} \mathbf{f}_{j_{2}}+\cdots+\mathbf{H}_{k 2} \mathbf{f}_{j_{k}}, \\
\vdots \\
\mathbf{f}_{d} \approx \mathbf{X W h}_{d}=\mathbf{H}_{1 d} \mathbf{f}_{j_{1}}+\mathbf{H}_{2 d} \mathbf{f}_{j_{2}}+\cdots+\mathbf{H}_{k d} \mathbf{f}_{j_{k}},
\end{array}\right. \label{211}
\end{equation}
where \(\mathbf{H} \in \mathbb{R}^{k\times d}\) is referred to as the representation matrix and its \(i\)th column denoted by \(\mathbf{h}_{i}\), for \(i = 1, \dots, d\). The entries of \(\mathbf{H}\) are denoted by \(\mathbf{H}_{ij}\), for \(i = 1, \dots, k\) and \(j = 1, \dots, d\).

In summary, we use the linear combinations of the selected features to approximate the high-dimensional input data matrix, i.e., \(\mathbf{X}\approx \mathbf{X}\mathbf{W}\mathbf{H}\). This is the basis for the introduction of a novel UFS method, namely MFFS \cite{wang2015subspace} which utilizes the Frobenius norm to evaluate the distance between the spaces spanned by \(\mathbf{X}\) and \(\mathbf{X}\mathbf{W}\mathbf{H}\), ultimately transformed into a non-negative matrix factorization problem as:
\begin{equation}
\begin{aligned} 
& \arg \min _{\mathbf{W}, \mathbf{H}} \|\mathbf{X}-\mathbf{X} \mathbf{W} \mathbf{H}\|_F
\\ & \quad \text { subject to } \mathbf{W} \geq 0, \quad \mathbf{H} \geq 0, \quad \mathbf{W} \text{ is an indicator matrix.}
\end{aligned}\label{212}
\end{equation}

Some subsequent studies \cite{wang2015unsupervised, qi2018unsupervised, shang2019local, li2022unsupervised} have made improvements within this framework. However, the Frobenius norm measuring the distance between the space formed by the original space and the selected features fails to consider nonlinear relationships between features and is sensitive to noise and outliers \cite{nie2010efficient}. Given these considerations, in this section, we propose a kernel alignment based learning method for unsupervised feature selection, termed KAUFS. The KAUFS method characterizes the similarity between two kernels, one associated with the original features and the other with the selected features, enabling the selection of highly discriminative features which consider the nonlinear relationships among features.
\subsection{Kernel Alignment}

We briefly review several concepts related to centered kernel and kernel alignment following \cite{cortes2012algorithms}. Suppose that the data matrix \(\mathbf{X}=\left[\mathbf{x}_{1}; \mathbf{x}_{2}; \dots ;\mathbf{x}_{n}\right] \in \mathbb{R}^{n\times d}\). A mapping function \(\Phi : \mathbb{R}^{d} \to \mathcal{H}\) which maps the data samples from the input space to a reproducing kernel Hilbert space \(\mathcal{H}\). Thus, we have the transformed data sample \(\mathbf{x}\) from the original space \(\mathbf{X}\) to the kernel space is \(\Phi\left(\mathbf{x}\right)\). In addition, the mapping can be centered by subtracting its empirical expectation, that is forming it by \(\Phi\left(\mathbf{x}\right)-\bar{\Phi}\), where \(\bar{\Phi}=\frac{1}{n}\sum_{i=1}^{n}\Phi\left(\mathbf{x}_{i}\right)\). 

Consider a positive definite kernel function \(k:\mathbb{R}^{d} \times \mathbb{R}^{d} \to \mathbb{R}\) which can be induced by the mapping function \(\Phi\) with entries, \(k\left(\mathbf{x}, \mathbf{x}'\right)=\left \langle \Phi\left(\mathbf{x}\right),\Phi\left(\mathbf{x}'\right) \right \rangle=\Phi\left(\mathbf{x}\right)^{T}\Phi\left(\mathbf{x}'\right)\). Centering the positive definite kernel function consists of centering all the mapping \(\Phi\) associated with \(k\).

\begin{definition}[Centered Kernel Function]
For any two samples \(\mathbf{x}\) and \(\mathbf{x}'\) in the data matrix \(\mathbf{X}\), the positive definite kernel function \(k\left(\mathbf{x}, \mathbf{x}'\right)\) after centering is
\begin{equation}
\begin{aligned}
k_{c}\left(\mathbf{x}, \mathbf{x}'\right)&=\left(\Phi\left(\mathbf{x}\right)-\bar{\Phi}\right)^{T}\left(\Phi\left(\mathbf{x}'\right)-\bar{\Phi}\right)\\
&=k\left(\mathbf{x}, \mathbf{x}'\right)-\frac{1}{n} \sum_{i=1}^{n} k\left(\mathbf{x}, \mathbf{x}_{i}\right)-\frac{1}{n} \sum_{i=1}^{n} k\left(\mathbf{x}_{i}, \mathbf{x}'\right)+\frac{1}{n^{2}}\sum_{i=1}^{n}\sum_{j=1}^{n}k\left(\mathbf{x}_{i}, \mathbf{x}_{j}\right) 
\end{aligned}.\nonumber
\end{equation}
\end{definition}

\begin{definition}[Centered Kernel Matrix]
The positive semidefinite kernel matrix \(\mathbf{K}\), associated with the positive definite kernel function \(k\) for all samples in the data matrix \(\mathbf{X}\), where the sample index \(i,j \in \left[1,n\right]\), after centering is
\begin{equation}
\left[\mathbf{K}_{c}\right]_{ij}=\mathbf{K}_{ij}-\frac{1}{n} \sum_{j=1}^{n} \mathbf{K}_{ij}-\frac{1}{n} \sum_{i=1}^{n} \mathbf{K}_{ij}+\frac{1}{n^{2}}\sum_{i=1}^{n}\sum_{j=1}^{n}\mathbf{K}_{ij}.\nonumber
\end{equation}
\end{definition}

Denoting \(\mathbf{\Lambda} =\mathbf{I}_{n}-\frac{1}{n} \mathbf{1}_{n\times n}\), the centered kernel matrix \(\mathbf{K}_{c}\) can also be expressed as
\begin{equation}
\mathbf{K}_{c} = \mathbf{\Lambda}\mathbf{K}\mathbf{\Lambda}.\label{223}
\end{equation}

\begin{definition}[Centered Kernel Alignment]
Let \(k\) and \(k'\) denote two positive definite kernel functions defined over \(\mathbb{R}^{d} \times \mathbb{R}^{d}\), and \(\mathbf{K}\) and \(\mathbf{K}'\) represent their kernel matrices, satisfying \(\|\mathbf{K}_c\|_F \ne 0\) and \(\|\mathbf{K}'_c\|_F \ne 0\). Then, the alignment \(\rho \left(k,k'\right)\) between \(k\) and \(k'\) and the alignment \(\hat{\rho}\left(\mathbf{K},\mathbf{K}'\right)\) between \(\mathbf{K}\) and \(\mathbf{K}'\) are defined by
\begin{equation}
\rho\left(k, k'\right)=\frac{\mathrm{E}\left[k_{c} k_{c}^{\prime}\right]}{\sqrt{\mathrm{E}\left[k_{c}^{2}\right] \mathrm{E}\left[k_{c}^{\prime 2}\right]}} \quad \text{and} \quad \hat{\rho}\left(\mathbf{K}, \mathbf{K}'\right)=\frac{\operatorname{Tr}\left(\mathbf{K}_{c} \mathbf{K}'_{c}\right)}{\left\|\mathbf{K}_{c}\right\|_{F} \left\|\mathbf{K}'_{c}\right\|_{F}}.\nonumber
\end{equation}
\end{definition}

The centered kernel matrix alignment \(\hat{\rho}\left(\mathbf{K}, \mathbf{K}'\right)\) can be viewed as calculating the cosine similarity between two kernel matrices. In addition, in many cases, incorporating the normalization term \(\left\|\mathbf{K}_{c}\right\|_{F} \left\|\mathbf{K}'_{c}\right\|_{F}\)  typically adds complexity to the optimization problem. Hence, a more widely adopted approach to kernel alignment is its unnormalized version \cite{wei2016nonlinear}.

\begin{definition}[Unnormalized Centered Kernel Alignment]
Let \(k\) and \(k'\) denote two positive definite kernel functions defined over \(\mathbb{R}^{d} \times \mathbb{R}^{d}\), and \(\mathbf{K}\) and \(\mathbf{K}'\) represent their kernel matrices for a sample of size \(n\). Then, the unnormalized alignment \(\rho_{u} \left(k,k'\right)\) between \(k\) and \(k'\) and the unnormalized alignment \(\hat{\rho}_{u}\left(\mathbf{K},\mathbf{K}'\right)\) between \(\mathbf{K}\) and \(\mathbf{K}'\) are defined by
\begin{equation}
\rho_{u}\left(k, k'\right)=\mathrm{E}\left[k_{c} k_{c}^{\prime}\right]  \quad \text{and} \quad \hat{\rho}_{u}\left(\mathbf{K}, \mathbf{K}'\right)=\frac{1}{n^{2}}\operatorname{Tr}\left(\mathbf{K}_{c} \mathbf{K}'_{c}\right).\nonumber
\end{equation}
\end{definition}

It is worthy to note that since the kernel alignment is only evaluated using positive semidefinite kernel matrices, UFS methods proposed in this paper are not suitable for utilizing non-positive semidefinite kernel matrices.

The main idea of our proposed method is to seek a dataset composed of a concise set of features that aligns best with the original dataset in the kernel space. The objective is to maximize the similarity between the kernel matrices corresponding to the original features and the selected features, ensuring that the selected features retain the information from the original features to the greatest extent. Specifically, given a high-dimensional input data matrix \(\mathbf{X}=\left[\mathbf{x}_{1}; \mathbf{x}_{2}; \dots ;\mathbf{x}_{n}\right] \in \mathbb{R}^{n\times d}\) and a positive definite kernel function \(k:\mathbb{R}^{d} \times \mathbb{R}^{d} \to \mathbb{R}\), we can obtain the kernel matrix \(\mathbf{K}\) corresponding to the original features:
\begin{equation}
\mathbf{K}=\left(\begin{array}{ccc}k\left(\mathbf{x}_1, \mathbf{x}_1\right) & \cdots & k\left(\mathbf{x}_1, \mathbf{x}_n\right) \\ \vdots & \ddots & \vdots \\ k\left(\mathbf{x}_n, \mathbf{x}_1\right) & \cdots & k\left(\mathbf{x}_n,\mathbf{x}_n\right)\end{array}\right).\nonumber
\end{equation}

For the dataset \(\mathbf{X}\mathbf{W}\mathbf{H}\) formed by the selected features through the model, we compute the linear kernel matrix corresponding to this dataset. We desire the linear kernel matrix of the selected features and the kernel matrix of the original features to exhibit similar perspectives on the similarity among samples. Therefore, our objective is to select features that maximize the following unnormalized centered kernel matrix alignment expression
\begin{equation}
\frac{1}{n^{2}}\operatorname{Tr}\left(\mathbf{\Lambda}\mathbf{K}\mathbf{\Lambda}\mathbf{\Lambda}\mathbf{X}\mathbf{W}\mathbf{H}\mathbf{H}^T\mathbf{W}^T\mathbf{X}^T\mathbf{\Lambda}\right)=\frac{1}{n^{2}}\operatorname{Tr}\left(\mathbf{\Lambda}\mathbf{K}\mathbf{\Lambda}\mathbf{X}\mathbf{W}\mathbf{H}\mathbf{H}^T\mathbf{W}^T\mathbf{X}^T\right),\label{221}
\end{equation}
where the second equation can be obtain by the fact that \(\mathbf{\Lambda}\mathbf{\Lambda}=\mathbf{\Lambda}\) and the trace cyclic property.

Similar to the problem \eqref{212}, according to the centered kernel matrix formula \eqref{223} and the kernel alignment function \eqref{221}, the primary formulation of our feature selection method can be expressed as the following matrix factorization problem: 
\begin{equation}
\begin{aligned} 
& \arg \min _{\mathbf{W}, \mathbf{H}} -\operatorname{Tr}\left(\mathbf{K}_{c}\mathbf{X}\mathbf{W}\mathbf{H}\mathbf{H}^{T}\mathbf{W}^{T}\mathbf{X}^{T}\right) 
\\ & \quad \text { subject to } \mathbf{W} \geq 0, \quad \mathbf{H} \geq 0, \quad \mathbf{W} \text{ is an indicator matrix.}
\end{aligned}\label{222}
\end{equation}
\subsection{KAUFS Model}

In order for relaxing the feature weight matrix \(\mathbf{W}\) from being an indicator matrix, Wang et al. \cite{wang2015subspace} initially proposed that \(\mathbf{W}\) should satisfy \(\mathbf{W}^{T}\mathbf{W}=\mathbf{I}_{k}\). However, Saberi-Movahed et al. \cite{saberi2022dual} pointed out that this indicator constraint on \(\mathbf{W}\) may not necessarily yield an indicator matrix. Moreover, such constraint may not perform effectively in identifying salient features. Inspired by Han et al. \cite{han2015selecting}, our proposed algorithm for matrix factorization-based feature selection incorporates a regularization method known as inner product regularization. This regularization is employed to characterize both sparsity and low redundancy simultaneously in the variables. Thus, we introduce the inner product regularization for learning the feature weight matrix \(\mathbf{W}\), defined as follows
\begin{equation}
\operatorname{Reg}(\mathbf{W})=\sum_{i=1}^{d}\sum_{j=1,j\ne i}^{d}  \left\langle\mathbf{w}_i, \mathbf{w}_j\right\rangle=\operatorname{Tr}\left(\mathbf{1}_{d\times d} \mathbf{W} \mathbf{W}^T\right)-\operatorname{Tr}\left(\mathbf{W} \mathbf{W}^T\right).\label{231}
\end{equation}

Recently, researchers have incorporated a regularization strategy termed dual sparsity inner product regularization into the feature selection model to enhance the selection of more prominent features \cite{saberi2022dual}. That is, the inner product regularization should not only apply to \(\mathbf{W}\) but also to the representation matrix \(\mathbf{H}\). This consideration arises from the fact that, according to the equation group \eqref{211}, we can obtain the \(i\)-th original feature in \(\mathbf{X}\) approximation expression
\begin{equation}
\mathbf{f}_{i} \approx \mathbf{X W h}_{i}=  \mathbf{H}_{1i} \mathbf{f}_{j_{1}}+\mathbf{H}_{2i} \mathbf{f}_{j_{2}}+\cdots+\mathbf{H}_{k i} \mathbf{f}_{j_{k}}. \label{232}
\end{equation}
The \(i\)-th column in \(\mathbf{H}\) given in \eqref{232} plays an inevitable role in the self-representation property. In other words, when two original features \(\mathbf{f}_{i}\) and \(\mathbf{f}_{j}\) are highly dependent, it's reasonable to expect that their corresponding columns in the representation matrix which denoted as \(\mathbf{h}_{i}\) and \(\mathbf{h}_{j}\) should exhibit similar patterns. The inner product regularization applied to the representation matrix not only induces sparsity in the columns of \(\mathbf{H}\) but also shows beneficial in identifying redundant features, preserving those with discriminative qualities. Thus, leveraging the advantages of inner product regularization, a regularization on \(\mathbf{H}\) can be designed as
\begin{equation}
\operatorname{Reg}(\mathbf{H})=\sum_{i=1}^{d}\sum_{j=1,j\ne i}^{d}  \left\langle\mathbf{h}_i, \mathbf{h}_j\right\rangle=\operatorname{Tr}\left(\mathbf{1}_{d\times d} \mathbf{H}^T \mathbf{H}\right)-\operatorname{Tr}\left(\mathbf{H}^T \mathbf{H}\right).\label{233}
\end{equation}

In summary, by substituting the indicator constraints for \(\mathbf{W}\) with the inner product regularization provided by \eqref{231}, and integrating the regularization term on \(\mathbf{H}\) given by \eqref{233} into the problem \eqref{222}, the objective function of our KAUFS method is formulated as:
\begin{equation}
\begin{aligned} 
& \arg \min _{\mathbf{W}, \mathbf{H}} -\frac{1}{2}\operatorname{Tr}\left(\mathbf{K}_{c}\mathbf{X}\mathbf{W}\mathbf{H}\mathbf{H}^{T}\mathbf{W}^{T}\mathbf{X}^{T}\right)\\
& \hphantom{\arg \min _{\mathbf{W}, \mathbf{H}}} + \frac{\alpha}{2} \left[\operatorname{Tr}\left(\mathbf{1}_{d\times d} \mathbf{W} \mathbf{W}^T\right)-\operatorname{Tr}\left(\mathbf{W} \mathbf{W}^T\right)\right]\\
& \hphantom{\arg \min _{\mathbf{W}, \mathbf{H}}} + \frac{\beta}{2} \left[\operatorname{Tr}\left(\mathbf{1}_{d\times d} \mathbf{H}^T \mathbf{H}\right)-\operatorname{Tr}\left(\mathbf{H}^T \mathbf{H}\right)\right]\\ 
& \quad \text { subject to } \mathbf{W} \geq 0, \quad \mathbf{H} \geq 0
\end{aligned}\label{234}
\end{equation}
where \(\alpha\) and \(\beta\) are two regularization parameters to balance these three terms.

As a summary of what is discussed in this section, on one hand, in order to enable the model to consider nonlinear relationships among features, we utilize kernel alignment instead of traditional entry-wise matrix norms such as Frobenius norm, \(l_{2,1}\) norm, etc., as a subspace learning distance. On the other hand, in order to tackle the sparsity of the feature weight matrix \(\mathbf{W}\) and characterize the correlation of features, we further apply inner product regularization to the feature weight matrix \(\mathbf{W}\) and the representation matrix \(\mathbf{H}\). The combination of the kernel alignment distance and the inner product regularization leads to a notable reduction in redundancy among the selected features.

\section{Algorithm and Convergence Analysis}\label{Algorithm and Convergence Analysis}
In this section, we propose an iterative update algorithm to solve the KAUFS model and analyze its convergence.

\subsection{Algorithm}

In order to solve the optimization problem \eqref{234}, we utilize the Lagrange multiplier method. For this purpose, we set Lagrange multipliers \(\mathbf{A} \in \mathbb{R}^{d\times k}\) and \(\mathbf{B} \in \mathbb{R}^{k\times d}\) for the constraints \(\mathbf{W} \ge 0\) and \(\mathbf{H} \ge 0\). We construct the Lagrange function as follows
\begin{equation}
\begin{aligned} 
\mathcal{L}\left(\mathbf{W},\mathbf{H},\mathbf{A},\mathbf{B}\right)=&-\frac{1}{2}\operatorname{Tr}\left(\mathbf{K}_{c}\mathbf{X}\mathbf{W}\mathbf{H}\mathbf{H}^{T}\mathbf{W}^{T}\mathbf{X}^{T}\right)\\
&+\frac{\alpha}{2} \left[\operatorname{Tr}\left(\mathbf{1}_{d \times d} \mathbf{W} \mathbf{W}^T\right)-\operatorname{Tr}\left(\mathbf{W} \mathbf{W}^T\right)\right]\\
&+\frac{\beta}{2} \left[\operatorname{Tr}\left(\mathbf{1}_{d \times d} \mathbf{H}^T \mathbf{H}\right)-\operatorname{Tr}\left(\mathbf{H}^T \mathbf{H}\right)\right]\\
&+\operatorname{Tr}\left(\mathbf{A}\mathbf{W}^T\right)+\operatorname{Tr}\left(\mathbf{B}\mathbf{H}^T\right).
\end{aligned}\nonumber
\end{equation}
Calculating the partial derivatives of  \(\mathcal{L}\) with respect to \(\mathbf{W}\) and \(\mathbf{H}\) respectively, we can obtain
\begin{equation}
\frac{\partial \mathcal{L}}{\partial \mathbf{W}}=-\mathbf{X}^T \mathbf{K}_c \mathbf{X}\mathbf{W}\mathbf{H}\mathbf{H}^T+\alpha \mathbf{1}_{d \times d}\mathbf{W}-\alpha\mathbf{W}+\mathbf{A}\label{311}
\end{equation}
and
\begin{equation}
\frac{\partial \mathcal{L}}{\partial \mathbf{H}}=-\mathbf{W}^T\mathbf{X}^T \mathbf{K}_c \mathbf{X}\mathbf{W}\mathbf{H}+\beta \mathbf{H}\mathbf{1}_{d \times d}-\beta\mathbf{H}+\mathbf{B}.\label{312}
\end{equation}

We noticed that equations \eqref{311} and \eqref{312} include the term \(\mathbf{X}^T \mathbf{K}_c\mathbf{X}\). If the data matrix \(\mathbf{X}\) or the centered kernel matrix \(\mathbf{K}_c\) contains negative values, this could result in the iterative update rules for \(\mathbf{W}\) and \(\mathbf{H}\) producing negative values, leading to \(\mathbf{W}\) and \(\mathbf{H}\) not satisfying the non-negative constraint. To address this issue, we employ the technique proposed by Ding et al. \cite{ding2008convex} to further refine equations \eqref{311} and \eqref{312}. That is, utilizing \(\mathbf{X}^T \mathbf{K}_c\mathbf{X}\) to construct two non-negative matrices, with each entry in these two matrices being
\begin{align}
\left(\mathbf{X}^T\mathbf{K}_c\mathbf{X}\right)_{ij}^+ &= \frac{\left|\left(\mathbf{X}^T\mathbf{K}_c\mathbf{X}\right)_{ij}\right| + \left(\mathbf{X}^T\mathbf{K}_c\mathbf{X}\right)_{ij}}{2},\label{313} \\
\left(\mathbf{X}^T\mathbf{K}_c\mathbf{X}\right)_{ij}^- &= \frac{\left|\left(\mathbf{X}^T\mathbf{K}_c\mathbf{X}\right)_{ij}\right| - \left(\mathbf{X}^T\mathbf{K}_c\mathbf{X}\right)_{ij}}{2}.\label{314}
\end{align}
It can be verified that \(\mathbf{X}^T\mathbf{K}_c\mathbf{X}=\left(\mathbf{X}^T\mathbf{K}_c\mathbf{X}\right)^+-\left(\mathbf{X}^T\mathbf{K}_c\mathbf{X}\right)^-\). Then, we have
\begin{equation}
\frac{\partial \mathcal{L}}{\partial \mathbf{W}}=-\left(\mathbf{X}^T \mathbf{K}_c \mathbf{X}\right)^{+}\mathbf{W}\mathbf{H}\mathbf{H}^T+\left(\mathbf{X}^T \mathbf{K}_c \mathbf{X}\right)^{-}\mathbf{W}\mathbf{H}\mathbf{H}^T+\alpha \mathbf{1}_{d \times d}\mathbf{W}-\alpha\mathbf{W}+\mathbf{A}\nonumber
\end{equation}
and
\begin{equation}
\frac{\partial \mathcal{L}}{\partial \mathbf{H}}=-\mathbf{W}^T\left(\mathbf{X}^T \mathbf{K}_c \mathbf{X}\right)^{+}\mathbf{W}\mathbf{H}+\mathbf{W}^T\left(\mathbf{X}^T \mathbf{K}_c \mathbf{X}\right)^{-}\mathbf{W}\mathbf{H}+\beta \mathbf{H}\mathbf{1}_{d \times d}-\beta\mathbf{H}+\mathbf{B}.\nonumber
\end{equation}
Finally, exploiting the Karush-Kuhn–Tucker conditions \(\mathbf{A}_{ij}\mathbf{W}_{ij}^{2}=0\) and \(\mathbf{B}_{ij}\mathbf{H}_{ij}^{2}=0\) leads to the following updating rules
\begin{equation}
\mathbf{W}_{i j} \gets \mathbf{W}_{i j} \sqrt{\frac{\left(\left(\mathbf{X}^T \mathbf{K}_c \mathbf{X}\right)^{+}\mathbf{W}\mathbf{H}\mathbf{H}^T +\alpha \mathbf{W}\right)_{i j}}{\left(\left(\mathbf{X}^T \mathbf{K}_c \mathbf{X}\right)^{-}\mathbf{W}\mathbf{H}\mathbf{H}^T+\alpha \mathbf{1}_{d \times d}\mathbf{W}\right)_{i j}}} \label{315}
\end{equation}
and
\begin{equation}
\mathbf{H}_{i j} \gets \mathbf{H}_{i j} \sqrt{\frac{\left(\mathbf{W}^T\left(\mathbf{X}^T \mathbf{K}_c \mathbf{X}\right)^{+}\mathbf{W}\mathbf{H} +\beta \mathbf{H}\right)_{i j}}{\left(\mathbf{W}^T\left(\mathbf{X}^T \mathbf{K}_c \mathbf{X}\right)^{-}\mathbf{W}\mathbf{H}+\beta \mathbf{H}\mathbf{1}_{d \times d}\right)_{i j}}}. \label{316}
\end{equation}

Based on the above analysis, the entire framework of the proposed KAUFS method is  summarized in Algorithm \ref{Algorithm KAUFS}.
\begin{algorithm}[H]
\caption{Kernel Alignment Unsupervised Feature Selection (KAUFS)}
\label{Algorithm KAUFS}
\begin{algorithmic}[1]
\State \textbf{Input} Data matrix $\mathbf{X} \in \mathbb{R}^{n\times d}$, the centered kernel matrix $\mathbf{K}_c \in \mathbb{R}^{n\times n}$, the number of selected features $k$ and two regularization parameters $\alpha$ and $\beta$;
\State Initialize $\mathbf{W} \in \mathbb{R}^{d\times k}$ and $\mathbf{H} \in \mathbb{R}^{k\times d}$;
\State Compute $\left(\mathbf{X}^T\mathbf{K}_c\mathbf{X}\right)^+$ and $\left(\mathbf{X}^T\mathbf{K}_c\mathbf{X}\right)^-$ using the rules \eqref{313} and \eqref{314};
\State \textbf{repeat}
\State \qquad Update $\mathbf{W}_{ij}$ by the rule \eqref{315};
\State \qquad Update $\mathbf{H}_{ij}$ by the rule \eqref{316};
\State \textbf{until} Convergence criterion has been satisfied;
\State \textbf{Output} Sort all $d$ features in descending order according to the value of $\left\|\mathbf{w}^{i}\right\|_{2}$, $i=1, 2,\dots, d$. The top $k$ features of $\mathbf{X}$ are selected based on these top $k$ highest $l_{2}$-norm values of $\mathbf{W}$ to form the optimal feature subset, and the corresponding row indexes are regarded as the output of the KAUFS algorithm.
\end{algorithmic}
\end{algorithm}
\subsection{Convergence Analysis}
In the following, we study the convergence of the proposed iterative update rules in Algorithm \ref{Algorithm KAUFS}.

\begin{theorem}
For \(\mathbf{W},\mathbf{H} \ge 0\), the values of the objective function given in \eqref{234} are non-increasing by employing the updating rules \eqref{315} and \eqref{316}.
\end{theorem}
\begin{proof}
The objective function of the KAUFS method can be written as
\begin{equation}
\begin{aligned} 
\mathcal{J}\left(\mathbf{W},\mathbf{H}\right)=&-\frac{1}{2}\operatorname{Tr}\left(\mathbf{K}_{c}\mathbf{X}\mathbf{W}\mathbf{H}\mathbf{H}^{T}\mathbf{W}^{T}\mathbf{X}^{T}\right)\\
&+\frac{\alpha}{2} \left[\operatorname{Tr}\left(\mathbf{1}_{d \times d} \mathbf{W} \mathbf{W}^T\right)-\operatorname{Tr}\left(\mathbf{W} \mathbf{W}^T\right)\right]\\
&+\frac{\beta}{2} \left[\operatorname{Tr}\left(\mathbf{1}_{d \times d} \mathbf{H}^T \mathbf{H}\right)-\operatorname{Tr}\left(\mathbf{H}^T \mathbf{H}\right)\right].
\end{aligned}\label{321}
\end{equation}

Proving this theorem consists of two steps. In the first step, we initially fix the matrix \(\mathbf{H}\) and look for an appropriate auxiliary function \cite{lee2000algorithms} for \(\mathcal{J}\left(\mathbf{W}\right)\) according to the equation \eqref{321}. Subsequently, we examine the descending behavior of \(\mathcal{J}\left(\mathbf{W}\right)\) utilizing the proposed auxiliary function. In the second step, we keep the matrix \(\mathbf{W}\) fixed and perform a similar analysis on the updating formula for \(\mathbf{H}\).

Let the matrix \(\mathbf{H}\) be fixed. By neglecting terms in equation \eqref{321} that remain constant or only depend on the matrix \(\mathbf{H}\), we can reformulate the function \(\mathcal{J}\left(\mathbf{W},\mathbf{H}\right)\) into the following form
\begin{equation}
\mathcal{J}\left(\mathbf{W}\right)=-\frac{1}{2}\operatorname{Tr}\left(\mathbf{K}_{c}\mathbf{X}\mathbf{W}\mathbf{H}\mathbf{H}^{T}\mathbf{W}^{T}\mathbf{X}^{T}\right) +\frac{\alpha}{2} \left[\operatorname{Tr}\left(\mathbf{1}_{d \times d} \mathbf{W} \mathbf{W}^T\right)-\operatorname{Tr}\left(\mathbf{W} \mathbf{W}^T\right)\right].
\nonumber
\end{equation}
Let \(\mathbf{A}=\mathbf{X}^T \mathbf{K}_c \mathbf{X}\) and \(\mathbf{B}=\mathbf{H}\mathbf{H}^{T}\). Utilizing the trace cyclic property and formulas \eqref{313} and \eqref{314}. The function \(\mathcal{J}\left(\mathbf{W}\right)\) can be further converted into
\begin{equation}
\begin{aligned}
\mathcal{J}\left(\mathbf{W}\right)=&-\frac{1}{2}\operatorname{Tr}\left(\mathbf{A}^{+}\mathbf{W}\mathbf{B}\mathbf{W}^{T}\right)+\frac{1}{2}\operatorname{Tr}\left(\mathbf{A}^{-}\mathbf{W}\mathbf{B}\mathbf{W}^{T}\right)\\
&+\frac{\alpha}{2} \operatorname{Tr}\left(\mathbf{1}_{d \times d} \mathbf{W} \mathbf{W}^T\right)-\frac{\alpha}{2}\operatorname{Tr}\left(\mathbf{W} \mathbf{W}^T\right).
\end{aligned}\nonumber
\end{equation}

Let us define the following function
\begin{equation}
\begin{aligned}
\mathcal{G}\left(\mathbf{W},\mathbf{W}'\right)=&-\frac{1}{2} \sum_{i=1}^{d}\sum_{j=1}^{d}\sum_{m=1}^{k}\sum_{n=1}^{k}\mathbf{A}^{+}_{ij}\mathbf{W}'_{jm}\mathbf{B}_{mn}\mathbf{W}'_{in}\left(1+\log\frac{\mathbf{W}_{jm}\mathbf{W}_{in}}{\mathbf{W}'_{jm}\mathbf{W}'_{in}} \right) \\
&+\frac{1}{2}\sum_{i=1}^{d}\sum_{j=1}^{k}\frac{\left(\mathbf{A}^{-}\mathbf{W}'\mathbf{B}\right)_{ij}\mathbf{W}^{2}_{ij}}{\mathbf{W}'_{ij}}\\
&+\frac{\alpha}{2}\sum_{i=1}^{d}\sum_{j=1}^{k}\frac{\left(\mathbf{1}_{d\times d}\mathbf{W'}\right)_{ij}\mathbf{W}^{2}_{ij}}{\mathbf{W}'_{ij}}\\
&-\frac{\alpha}{2}\sum_{i=1}^{d}\sum_{j=1}^{k}\mathbf{W'}^{2}_{ij}\left(1+\log \frac{\mathbf{W}^{2}_{ij}}{\mathbf{W'}^{2}_{ij}} \right),
\end{aligned}\nonumber
\end{equation}
where \(\mathbf{W}'\) is a non-negative matrix in \(\mathbb{R}^{d\times k}\). Our aim  is to demonstrate that \(\mathcal{G}\left(\mathbf{W},\mathbf{W}'\right)\) is an auxiliary function for \(\mathcal{J}\left(\mathbf{W}\right)\).  At first, it can be verified that when \(\mathbf{W}'=\mathbf{W}\), we have
\begin{equation}
\begin{aligned}
\mathcal{G}\left(\mathbf{W},\mathbf{W}\right)=&-\frac{1}{2} \sum_{i=1}^{d}\sum_{j=1}^{d}\sum_{m=1}^{k}\sum_{n=1}^{k}\mathbf{A}^{+}_{ij}\mathbf{W}_{jm}\mathbf{B}_{mn}\mathbf{W}_{in}\\
&+\frac{1}{2}\sum_{i=1}^{d}\sum_{j=1}^{k}\left(\mathbf{A}^{-}\mathbf{W}\mathbf{B}\right)_{ij}\mathbf{W}_{ij}\\
&+\frac{\alpha}{2}\sum_{i=1}^{d}\sum_{j=1}^{k}\left(\mathbf{1}_{d\times d}\mathbf{W}\right)_{ij}\mathbf{W}_{ij}\\
&-\frac{\alpha}{2}\sum_{i=1}^{d}\sum_{j=1}^{k}\mathbf{W}^{2}_{ij}\\
=&-\frac{1}{2}\operatorname{Tr}\left(\mathbf{A}^{+}\mathbf{W}\mathbf{B}\mathbf{W}^{T}\right)+\frac{1}{2}\operatorname{Tr}\left(\mathbf{A}^{-}\mathbf{W}\mathbf{B}\mathbf{W}^{T}\right)\\
&+\frac{\alpha}{2} \operatorname{Tr}\left(\mathbf{1}_{d \times d} \mathbf{W} \mathbf{W}^T\right)-\frac{\alpha}{2}\operatorname{Tr}\left(\mathbf{W} \mathbf{W}^T\right),
\end{aligned}\nonumber
\end{equation}
which implies that \(\mathcal{G}\left(\mathbf{W},\mathbf{W}\right)=\mathcal{J}\left(\mathbf{W}\right)\). Then, we use the inequality \(z \ge 1+\log z\), which holds for any \(z \ge 0\), and  obtain
\begin{equation}
\begin{aligned}
\operatorname{Tr}\left(\mathbf{A}^{+}\mathbf{W}\mathbf{B}\mathbf{W}^{T}\right)&=\sum_{i=1}^{d}\sum_{j=1}^{d}\sum_{m=1}^{k}\sum_{n=1}^{k}\mathbf{A}^{+}_{ij}\mathbf{W}_{jm}\mathbf{B}_{mn}\mathbf{W}_{in}\\
&\ge \sum_{i=1}^{d}\sum_{j=1}^{d}\sum_{m=1}^{k}\sum_{n=1}^{k}\mathbf{A}^{+}_{ij}\mathbf{W}'_{jm}\mathbf{B}_{mn}\mathbf{W}'_{in}\left(1+\log\frac{\mathbf{W}_{jm}\mathbf{W}_{in}}{\mathbf{W}'_{jm}\mathbf{W}'_{in}} \right)
\end{aligned}\label{322}
\end{equation}
and
\begin{equation}
\begin{aligned}
\operatorname{Tr}\left(\mathbf{W} \mathbf{W}^T\right)&=\sum_{i=1}^{d}\sum_{j=1}^{k}\mathbf{W}^{2}_{ij}\\
&\ge \sum_{i=1}^{d}\sum_{j=1}^{k}\mathbf{W'}^{2}_{ij}\left(1+\log \frac{\mathbf{W}^{2}_{ij}}{\mathbf{W'}^{2}_{ij}} \right).
\end{aligned}\label{323}
\end{equation}
Moreover, by utilizing the Proposition 6 in \cite{ding2006orthogonal}, the following inequalities can be derived
\begin{equation}
\operatorname{Tr}\left(\mathbf{A}^{-}\mathbf{W}\mathbf{B}\mathbf{W}^{T}\right) \le \sum_{i=1}^{d}\sum_{j=1}^{k}\frac{\left(\mathbf{A}^{-}\mathbf{W}'\mathbf{B}\right)_{ij}\mathbf{W}^{2}_{ij}}{\mathbf{W}'_{ij}} \label{324}
\end{equation}
and
\begin{equation}
\operatorname{Tr}\left(\mathbf{1}_{d \times d} \mathbf{W} \mathbf{W}^T\right) \le \sum_{j=1}^{k}\frac{\left(\mathbf{1}_{d\times d}\mathbf{W'}\right)_{ij}\mathbf{W}^{2}_{ij}}{\mathbf{W}'_{ij}}. \label{325}
\end{equation}
Combining the relations \eqref{322}, \eqref{323}, \eqref{324} and \eqref{325} yields that \(\mathcal{G}\left(\mathbf{W},\mathbf{W'}\right) \ge \mathcal{J}\left(\mathbf{W}\right)\).

Consequently, by considering relations \(\mathcal{G}\left(\mathbf{W},\mathbf{W}\right)=\mathcal{J}\left(\mathbf{W}\right)\) and \(\mathcal{G}\left(\mathbf{W},\mathbf{W'}\right) \ge \mathcal{J}\left(\mathbf{W}\right)\), we conclude that \(\mathcal{G}\left(\mathbf{W},\mathbf{W'}\right)\) serves as an auxiliary function for \(\mathcal{J}\left(\mathbf{W}\right)\), which means that \(\mathcal{J}\left(\mathbf{W}\right)\) is non-increasing by the update rule \(\arg \min _{\mathbf{W}} \mathcal{G}\left(\mathbf{W},\mathbf{W'}\right)\). In order to determine the minima of \(\mathcal{G}\left(\mathbf{W},\mathbf{W'}\right)\), by taking the derivatives of \(\mathcal{G}\) with respect to \(\mathbf{W}_{ij}\), it turns out that
\begin{equation}
\begin{aligned}
\frac{\partial \mathcal{G}\left(\mathbf{W},\mathbf{W'}\right)}{\partial \mathbf{W}_{ij}}=&-\left(\mathbf{A}^{+}\mathbf{W'}\mathbf{B}\right)_{ij}\frac{\mathbf{W'}_{ij}}{\mathbf{W}_{ij}} +\left(\mathbf{A}^{-}\mathbf{W'}\mathbf{B}\right)_{ij}\frac{\mathbf{W}_{ij}}{\mathbf{W'}_{ij}}\\
&+\alpha \left(\mathbf{1}_{d\times d}\mathbf{W'}\right)_{ij}\frac{\mathbf{W}_{ij}}{\mathbf{W'}_{ij}}-\alpha \mathbf{W'}_{ij}\frac{\mathbf{W'}_{ij}}{\mathbf{W}_{ij}}.
\end{aligned}\nonumber
\end{equation}
By setting \(\partial \mathcal{G}\left(\mathbf{W},\mathbf{W'}\right) / \partial \mathbf{W}_{ij}\) to zero, and substituting \(\left(\mathbf{X}^T \mathbf{K}_c \mathbf{X}\right)^{+}\), \(\left(\mathbf{X}^T \mathbf{K}_c \mathbf{X}\right)^{-}\) and \(\mathbf{H}\mathbf{H}^{T}\) into \(\mathbf{A}^{+}\), \(\mathbf{A}^{-}\) and \(\mathbf{B}\) respectively, we can derive
\begin{equation}
\mathbf{W}_{i j} = \mathbf{W'}_{i j} \sqrt{\frac{\left(\left(\mathbf{X}^T \mathbf{K}_c \mathbf{X}\right)^{+}\mathbf{W'}\mathbf{H}\mathbf{H}^T +\alpha \mathbf{W'}\right)_{i j}}{\left(\left(\mathbf{X}^T \mathbf{K}_c \mathbf{X}\right)^{-}\mathbf{W'}\mathbf{H}\mathbf{H}^T+\alpha \mathbf{1}_{d \times d}\mathbf{W'}\right)_{i j}}}.\nonumber
\end{equation}
This implies that when the matrix \(\mathbf{H}\) is fixed, the objective function of the KAUFS method exhibits non-increasing behavior with respect to the updating formula given by equation \eqref{315}.

The rest of the proof regarding the decreasing behavior of the objective function of KAUFS, with respect to the updating formula \eqref{316}, follows a similar process as described above, and its explanation is omitted.

\end{proof}

\section{Mutiple Kernel Method for KAUFS}\label{Mutiple Kernel Method for KAUFS}

We can apply the KAUFS model to learn the similarity information among data samples, thereby obtaining a discriminative subset of features. However, the performance of a kernel learning model typically depends on the choice of the kernel function, and the optimal kernel is typically unknown and computationally challenging to identify. Additionally, in cases where sample features contain heterogeneous information \cite{pavlidis2001gene, mariette2018unsupervised}, the sample size is large \cite{sonnenburg2006large, rakotomamonjy2007more}, or samples are unevenly distributed in a high-dimensional feature space \cite{zheng2006non}, using a single kernel function to map all samples may not be appropriate.

Multiple kernel learning has been proposed to address the above issues. It constructs several candidate kernel matrices and combines them into a consensus kernel matrix \cite{gonen2011multiple}. This consensus kernel matrix helps accurately characterize the internal structure of the dataset and can achieve superior performance compared to a single kernel matrix or a combination of different kernels with different parameters for the same kernel function. Moreover, the consensus kernel matrix demonstrate enhanced robustness, minimizing the impact of noise or outliers and enhancing the stability of the model \cite{kang2017kernel}.

In this section, we extend the KAUFS model to a multiple kernel model and propose a corresponding solving algorithm. By incorporating a sufficient number and diversity of kernel functions, the algorithm can jointly learns the optimal convex combination of kernel matrices along with the feature weight matrix \(\mathbf{W}\) and the representation matrix \(\mathbf{H}\). At the end of this section, we provide the computational complexity analysis of the KAUFS algorithm and the MKAUFS algorithm respectively.

\subsection{MKAUFS Model}

Given a dataset \(\mathbf{X} \in \mathbb{R}^{n\times d}\) with \(n\) samples, each having \(d\) features, we suppose there are \(N\) different kernel functions \(\left \{ k_{i} \right \}_{i=1}^{N}\) available for the unsupervised feature selection task. Accordingly, there are \(N\) different kernel spaces. Typically, the most suitable kernel space is unknown. An intuitive way to use them is to concatenate all by concatenating all feature spaces into an augmented Hilbert space and associating each feature space with a relevance weight \(\eta_{i}\), \(i=1, 2, \dots, N\). More specifically, we apply these \(N\) kernel functions to the dataset, generating \(N\) different kernel matrices, and centralize them using \eqref{223}. Subsequently, we consider a convex combination of these centered kernel matrices \(\left \{ K_{c}^{(i)} \right \}_{i=1}^{N}\), i.e.
\begin{equation}
\mathbf {K}=\sum_{i=1}^{N}\eta_{i}\mathbf{K}_{c}^{(i)}, \quad \text { subject to } \sum_{i=1}^{N}\eta_{i}=1, \quad\eta_{i} \geq 0.\nonumber
\end{equation}
We substitute this relation into \eqref{234} to obtain the objective function for Multiple Kernel Alignment Unsupervised Feature Selection (MKAUFS):
\begin{equation}
\begin{aligned} 
& \arg \min _{\mathbf{W}, \mathbf{H}} -\frac{1}{2}\operatorname{Tr}\left(\sum_{i=1}^{N}\eta_{i}\mathbf{K}_{c}^{(i)}\mathbf{X}\mathbf{W}\mathbf{H}\mathbf{H}^{T}\mathbf{W}^{T}\mathbf{X}^{T}\right)\\
& \hphantom{\arg \min _{\mathbf{W}, \mathbf{H}}} + \frac{\alpha}{2} \left[\operatorname{Tr}\left(\mathbf{1}_{d\times d} \mathbf{W} \mathbf{W}^T\right)-\operatorname{Tr}\left(\mathbf{W} \mathbf{W}^T\right)\right]\\
& \hphantom{\arg \min _{\mathbf{W}, \mathbf{H}}} + \frac{\beta}{2} \left[\operatorname{Tr}\left(\mathbf{1}_{d\times d} \mathbf{H}^T \mathbf{H}\right)-\operatorname{Tr}\left(\mathbf{H}^T \mathbf{H}\right)\right]\\
& \hphantom{\arg \min _{\mathbf{W}, \mathbf{H}}} + \frac{\gamma}{2}\left \| \mathbf{\eta} \right \|_{2}^{2}  \\ 
& \quad \text { subject to } \mathbf{W} \geq 0, \quad \mathbf{H} \geq 0, \quad \sum_{i=1}^{N}\eta_{i}=1, \quad \eta_{i} \geq 0.
\end{aligned}\label{411}
\end{equation}
We denote the kernel weights as a vector \(\mathbf{\eta}=\left [\eta_{1},\dots, \eta_{N} \right ]^{T} \) and further impose the \(l_{2}\)-norm to it as the regularization term \(\left \| \mathbf{\eta} \right \|_{2}^{2}\) to prevent the weights from overfitting to one kernel.
\subsection{Algorithm}

In this subsection, we propose an algorithm to solve \eqref{411}. Similar to the KAUFS method, we also adopt an iterative strategy to alternately optimize \(\left(\mathbf{W}, \mathbf{H}\right)\) and \(\mathbf{\eta}\), while holding the other variable as constant.

\textbf{(1) Optimizing with respect to \(\left(\mathbf{W}, \mathbf{H}\right)\) when \(\mathbf{\eta}\) is fixed.} The optimization problem \eqref{411} is reduced to
\begin{equation}
\begin{aligned} 
& \arg \min _{\mathbf{W}, \mathbf{H}} -\frac{1}{2}\operatorname{Tr}\left(\mathbf{K}\mathbf{X}\mathbf{W}\mathbf{H}\mathbf{H}^{T}\mathbf{W}^{T}\mathbf{X}^{T}\right)\\
& \hphantom{\arg \min _{\mathbf{W}, \mathbf{H}}} + \frac{\alpha}{2} \left[\operatorname{Tr}\left(\mathbf{1}_{d\times d} \mathbf{W} \mathbf{W}^T\right)-\operatorname{Tr}\left(\mathbf{W} \mathbf{W}^T\right)\right]\\
& \hphantom{\arg \min _{\mathbf{W}, \mathbf{H}}} + \frac{\beta}{2} \left[\operatorname{Tr}\left(\mathbf{1}_{d\times d} \mathbf{H}^T \mathbf{H}\right)-\operatorname{Tr}\left(\mathbf{H}^T \mathbf{H}\right)\right]\\ 
& \quad \text { subject to } \mathbf{W} \geq 0, \quad \mathbf{H} \geq 0
\end{aligned}\nonumber
\end{equation}
where \(\mathbf {K}=\sum_{i=1}^{N}\eta_{i}\mathbf{K}_{c}^{(i)}\). Similar to the optimization of \(\mathbf{W}\) and \(\mathbf{H}\) of KAUFS method, we have following rules to update \(\mathbf{W}\) and \(\mathbf{H}\) respectively.
\begin{equation}
\mathbf{W}_{i j} \gets \mathbf{W}_{i j} \sqrt{\frac{\left(\left(\mathbf{X}^T \mathbf{K} \mathbf{X}\right)^{+}\mathbf{W}\mathbf{H}\mathbf{H}^T +\alpha \mathbf{W}\right)_{i j}}{\left(\left(\mathbf{X}^T \mathbf{K} \mathbf{X}\right)^{-}\mathbf{W}\mathbf{H}\mathbf{H}^T+\alpha \mathbf{1}_{d \times d}\mathbf{W}\right)_{i j}}} \label{421}
\end{equation}
and
\begin{equation}
\mathbf{H}_{i j} \gets \mathbf{H}_{i j} \sqrt{\frac{\left(\mathbf{W}^T\left(\mathbf{X}^T \mathbf{K} \mathbf{X}\right)^{+}\mathbf{W}\mathbf{H} +\beta \mathbf{H}\right)_{i j}}{\left(\mathbf{W}^T\left(\mathbf{X}^T \mathbf{K} \mathbf{X}\right)^{-}\mathbf{W}\mathbf{H}+\beta \mathbf{H}\mathbf{1}_{d \times d}\right)_{i j}}}. \label{422}
\end{equation}

\textbf{(2) Optimizing with respect to \(\mathbf{\eta}\) when \(\mathbf{W}\) and \(\mathbf{H}\) are fixed.} Removing the  irrelevant terms, the optimization problem \eqref{411} becomes
\begin{equation}
\begin{aligned} 
& \arg \min _{\mathbf{\eta}} -\frac{1}{2}\operatorname{Tr}\left(\sum_{i=1}^{N}\eta_{i}\mathbf{K}_{c}^{(i)}\mathbf{X}\mathbf{W}\mathbf{H}\mathbf{H}^{T}\mathbf{W}^{T}\mathbf{X}^{T}\right) + \frac{\gamma}{2}\left \| \mathbf{\eta} \right \|_{2}^{2}=  -\frac{1}{2}\sum_{i=1}^{N}\eta_{i}f_{i} + \frac{\gamma}{2}\sum_{i=1}^{N}\eta_{i}^{2}\\ 
& \quad \text { subject to } \sum_{i=1}^{N}\eta_{i}=1, \quad \eta_{i} \geq 0
\end{aligned},\label{423}
\end{equation}
where 
\begin{equation}
f_{i}=\operatorname{Tr}\left(\mathbf{K}_{c}^{(i)}\mathbf{X}\mathbf{W}\mathbf{H}\mathbf{H}^{T}\mathbf{W}^{T}\mathbf{X}^{T}\right).\label{424}
\end{equation}

The optimization of \eqref{423} with respect to the kernel weights \(\mathbf{\eta}\) could be solved as a standard quadratic programming (QP) problem. In Algorithm \ref{Algorithm MKAUFS}, we provide a complete algorithm for solving the problem \eqref{411}.
\begin{algorithm}[H]
\caption{Multiple Kernel Alignment Unsupervised Feature Selection (MKAUFS)}
\label{Algorithm MKAUFS}
\begin{algorithmic}[1]
\State \textbf{Input} Data matrix $\mathbf{X} \in \mathbb{R}^{n\times d}$, A set of centered kernel matrices $\left \{ K_{c}^{(i)} \right \}_{i=1}^{N}$, the number of selected features $k$ and three regularization parameters $\alpha$, $\beta$ and $\gamma$;
\State Initialize $\mathbf{W} \in \mathbb{R}^{d\times k}$ and $\mathbf{H} \in \mathbb{R}^{k\times d}$;
\State Initialize the kernel weights as $\eta_{i}=\frac{1}{N}$, $i=1,2,\dots, N$;
\State \textbf{repeat}
\State \qquad Compute the combined kernel $\mathbf {K}=\sum_{i=1}^{N}\eta_{i}\mathbf{K}_{c}^{(i)}$;
\State \qquad Compute $\left(\mathbf{X}^T\mathbf{K}\mathbf{X}\right)^+$ and $\left(\mathbf{X}^T\mathbf{K}\mathbf{X}\right)^-$ using the rules \eqref{313} and \eqref{314};
\State \qquad Update $\mathbf{W}_{ij}$ by the rule \eqref{421};
\State \qquad Update $\mathbf{H}_{ij}$ by the rule \eqref{422};
\State \qquad Compute $f_{i}$ for $i=1,2,\dots, N$ by the rule \eqref{424};
\State \qquad Update the kernel weights \(\mathbf{\eta}\) by \eqref{423};
\State \textbf{until} Convergence criterion has been satisfied;
\State \textbf{Output} Sort all $d$ features in descending order according to the value of $\left\|\mathbf{w}^{i}\right\|_{2}$, $i=1, 2,\dots, d$. The top $k$ features of $\mathbf{X}$ are selected based on these top $k$ highest $l_{2}$-norm values of $\mathbf{W}$ to form the optimal feature subset, and the corresponding row indexes are regarded as the output of the MKAUFS algorithm.
\end{algorithmic}
\end{algorithm}
\subsection{Computational Complexity}
The computational complexity of KAUFS and MKAUFS is illustrated below. We first analyze the computational complexity of the KAUFS method. The primary computation cost of the update rule \eqref{315} for \(\mathbf{W}\) is the calculation of \(\left(\mathbf{X}^T \mathbf{K}_c \mathbf{X}\right)\mathbf{W}\mathbf{H}\mathbf{H}^T\). It's worth noting that here we don't need to differentiate between \(\left(\mathbf{X}^T \mathbf{K}_c \mathbf{X}\right)^{+}\) and \(\left(\mathbf{X}^T \mathbf{K}_c \mathbf{X}\right)^{-}\), since the major computation cost in equations \eqref{313} and \eqref{314} is \(\left(\mathbf{X}^T \mathbf{K}_c \mathbf{X}\right)\). Assuming \(k, n \ll d\), we can obtain \(\left(\mathbf{X}^T \mathbf{K}_c \mathbf{X}\right)\mathbf{W}\mathbf{H}\mathbf{H}^T\) by first computing \(\mathbf{X}^T \mathbf{K}_c\) and \(\mathbf{H}\mathbf{H}^T\), then \(\mathbf{X}^T \mathbf{K}_c \mathbf{X}\) and \(\mathbf{W}\mathbf{H}\mathbf{H}^T\), and finally left-multiplying \(\mathbf{X}^T \mathbf{K}_c \mathbf{X}\) to \(\mathbf{W}\mathbf{H}\mathbf{H}^T\). This way, it takes \( 2dn^{2} + 4dk^{2} + 2nd^{2} + 2kd^{2} \) operations. Similarly, the cost of updating \(\mathbf{H}\) is mainly allocated to calculating \(\mathbf{W}^T\left(\mathbf{X}^T \mathbf{K}_c \mathbf{X}\right)\mathbf{W}\mathbf{H}\), which takes \( 2dn^{2} + 4kd^{2} + 2nd^{2} + 2dk^{2} \) operations. Consequently, the computational complexity for each iteration of Algorithm \ref{Algorithm KAUFS} is almost equal to \(O\left(d^{2}\right)\).

Next, we analyze the computational complexity of the MKAUFS method. The major computational cost of Algorithm \ref{Algorithm MKAUFS} are the updating rules \eqref{421}, \eqref{422}, and \eqref{424}. Assuming \(k, n, N \ll d\), the number of operations for updating rules \eqref{421} and \eqref{422} align with the conclusions drawn for Algorithm \ref{Algorithm KAUFS}. As for the equation \eqref{424}, we determine the optimal sequence of matrix multiplication. Firstly, we compute \(\mathbf{X}\mathbf{W}\), \(\mathbf{H}\mathbf{H}^T\), and \(\mathbf{W}^T\mathbf{X}^{T}\), then \(\mathbf{K}_{c}^{(i)}\mathbf{X}\mathbf{W}\) and \(\mathbf{H}\mathbf{H}^T\mathbf{W}^T\mathbf{X}^{T}\), and finally \(\mathbf{K}_{c}^{(i)}\mathbf{X}\mathbf{W}\mathbf{H}\mathbf{H}^{T}\mathbf{W}^{T}\mathbf{X}^{T}\). For \( N \) kernels, equation \eqref{424} requires \( N \) iterations, each taking \( 2dk^{2} + 4kdn + 4kn^{2} + 2nk^{2} \) operations, which is fewer than updating rules \eqref{421} and \eqref{422}. Thus, the computational complexity for each iteration of Algorithm \ref{Algorithm MKAUFS} is almost equal to \(O\left(d^{2}\right)\).

\section{Numerical Experiments}\label{Numerical Experiments}

In this section, we conduct experiments to demonstrate the effectiveness of our proposed algorithms in clustering applications using a variety of real-world datasets. Detailed information about the datasets utilized in this paper, the experimental settings, and the performance comparison of KAUFS and MKAUFS with other classic and state-of-the-art UFS methods are provided. Specifically, the comparison framework proceeds through three key steps: Firstly, the UFS algorithm is applied to the input dataset to select a number of representative features. Subsequently, the set of selected features and the desired number of classes are given as input to the k-means clustering algorithm to produce a vector containing predicted class labels for each sample. Finally, the performance of all UFS methods is assessed by computing three evaluation metrics using the predicted labels, true labels, and the selected features.

\subsection{Datasets}

Utilizing datasets from diverse fields provides a robust platform for assessing the effectiveness of different UFS methods. In this study, we used eight datasets including images, artificial data, and gene expression data, each detailed in Table \ref{Summary of the benchmark datasets}, to conduct computational experiments. Specifically, to assess performance on large-scale applied datasets, we utilized a high-dimensional face dataset (orlraws). These datasets are available in the scikit-feature selection repository \cite{li2017feature}.

\begin{table}[H]\centering
\caption {Summary of the benchmark datasets.}\label{Summary of the benchmark datasets}
\vspace{5pt}
\begin{tabular}{l|c|c|c|l}
\hline
Dataset & \# of Instances & \# of Features & \# of Classes & Data types\\
\hline
WarpAR & 130 & 2400 & 10 & Face image\\
WarpPIE & 210 & 2420 & 10 & Face image\\
Yale64 & 165 & 4096 & 15 & Face image\\
Orlraws & 100 & 10,304 & 10 & Face image\\
Madelon & 2600 & 500 & 2 & Artificial feature\\
GLIOMA & 50 & 4434 & 4 & Gene expression\\
TOX\_171 & 171 & 5748 & 4 & Gene expression\\
Prostate\_GE & 102 & 5966 & 2 & Gene expression\\
\hline
\end{tabular}
\end{table}
\subsection{Comparison Methods}
To evaluate the quality of our proposed methods, we compare our proposed KAUFS and MKAUFS methods with using all features and 6 other UFS methods as follows:

\begin{itemize}

\item Baseline: All original features are selected.

\item LS \cite{he2005laplacian}: Laplacian Score Feature Selection, which selects the features that can best preserve the local manifold structure of data.

\item KMFFS \cite{wang2015subspace}: Kernelized Matrix Factorization Feature Selection, which selects features through matrix factorization with an orthogonal constraint and leverages kernel functions to capture nonlinear relationships among features.

\item SCFS \cite{parsa2020unsupervised}: Subspace Clustering Unsupervised Feature Selection, which employs a self-expressive model to learn the clustering similarities in an adaptive manner to select features.

\item DRFSMFMR \cite{saberi2022dual}: Dual Regularized Unsupervised Feature Selection, which is based on matrix factorization and employs two inner product-based regularizations along with the global correlation regularization to achieve minimum redundancy in the feature subset.

\item VCSDFS \cite{karami2023unsupervised}: Variance–Covariance Subspace Distance Unsupervised Feature Selection, which considers variance and covariance information of the feature space to characterize a group of more representative features.

\item LS-CAE \cite{shaham2022deep}: Laplacian Score-regularized Concrete Autoencoder (CAE) based Feature Selection, which employs autoencoder architecture to cope with correlated features and utilizes Laplacian score criterion to keep away from the selection of nuisance features.

\end{itemize}
\subsection{Evaluation Metrics}
To evaluate the clustering performance of unsupervised feature selection, we employ two most prevalent measures, they are \cite{solorio2020review}: Clustering Accuracy (ACC) and Normalized Mutual Information (NMI). ACC measures the extent to which each cluster contain data points from the corresponding class. NMI reﬂects the consistency between clustering results and ground truth labels. Additionally, we construct a metric for measuring the selected features' redundancy rate based on Distance Correlation (DC) \cite{szekely2007measuring}, which allows the consideration of both linear and non-linear relationships between two variables. We name this metric the Redundancy Rate (RED). Here, it should be emphasized that the higher the value of ACC or NMI, the better the clustering performance is, which indicates that one can expect a feature selection method to be more successful in selecting more effective features. Additionally, the lower the value of RED, the more independent the features are within the selected subset.

The measure ACC is defined as follows:
\begin{equation}
ACC=\frac{\sum_{i=1}^{n} \delta\left(p_{i}, \operatorname{map}\left(q_{i}\right)\right)}{n},\nonumber
\end{equation}
where \(\delta\left(a,b\right)=1\) if \(a=b\) and 0, otherwise. \(p_{i}\) and \(q_{i}\) are the labels provided by  the dataset and the obtained clustering label, respectively. \(map(\bullet)\) is the permutation mapping function that matches the obtained clustering label to the equivalent label of the dataset using the Kuhn–Munkres algorithm. The NMI measure is characterized as:
\begin{equation}
N M I(P, Q)=\frac{I(P, Q)}{\sqrt{H(P) H(Q)}}, \nonumber
\end{equation}
where \(H(P)\) and \(H(Q)\) are the entropies of \(P\) and \(Q\), respectively, and \(I(P, Q)\) is the mutual information between \(P\) and \(Q\). For clustering, \(P\) and \(Q\) are the clustering results and the true labels, respectively.

Assume that \(\mathbf{F}\) is the set of selected features, which contain \(m\) features. Then RED is defined as:
\begin{equation}
RED\left(\mathbf{F}\right)=\frac{2}{m(m-1)}\sum_{i=1}^{m-1}\sum_{j=i+1}^{m}DC\left(f_{i}, f_{j}\right),\nonumber
\end{equation}
where \(DC\left(f_{i}, f_{j}\right)\) represents the distance correlation between feature \(f_{i}\) and \(f_{j}\), where \(f_{i}, f_{j}\in \mathbf{F}\).
\subsection{Experimental Setting}
There are certain parameters in the proposed methods KAUFS and MKAUFS, as well as in the other comparison methods mentioned above, whose values need to be predefined. The numbers of the selected features are taken as \(\{10t , t = 1, 2, 3,4,5,6,7,8,9,10\}\) for all datasets. In the Laplacian Score method, the parameter k in the built-in k-nearest neighbour algorithm is set to 5 for all datasets. As offered in \cite{wang2015subspace}, the penalty parameter \(\rho\) for the KMFFS method is set to \(10^8\). For the SCFS method, as stated in \cite{parsa2020unsupervised}, the parameter \(\gamma\) is fixed as \(10^6\), and the value of the other regularization parameters are searched in \(\{10^t , t = -4, -2, 0, 2, 4\}\). For the DRFSMFMR method, we follow the settings in the article\cite{saberi2022dual}, where the trade-off parameters \(\alpha\), \(\beta\), and \(\gamma\) are tuned in \(\{10^t, t = -3, -2, -1, 0, 1, 2, 3\}\). For the VCSDFS method, the value of the parameter \(\rho\) is varied from \(\{10^t, t = -5, -4, -3, -2, -1, 0, 1, 2, 3, 4, 5\}\). The settings related to the LS-CAE method are based on those given in \cite{shaham2022deep}. Finally, for our proposed methods KAUFS and MKAUFS, all trade-off parameters \(\alpha\), \(\beta\), and \(\gamma\) are tuned from \(\{10^t, t = -3, -2, -1, 0, 1, 2, 3\}\). 

In addition, we apply four widely used types of kernel functions, including linear, polynomial, Gaussian, and Laplacian kernels, to the KMFFS, KAUFS, and MKAUFS methods. Their expressions are as follows:
\begin{equation}
\begin{aligned}
\text{Linear kernel:} & \quad k\left(\mathbf{x}, \mathbf{y}\right) = \mathbf{x}^{T} \mathbf{y}. \\
\text{Polynomial kernel:} & \quad k\left(\mathbf{x}, \mathbf{y}\right) = \left(\mathbf{x}^{T} \mathbf{y} + c\right)^{d}. \\
\text{Gaussian kernel:} & \quad k\left(\mathbf{x}, \mathbf{y}\right) = \exp \left(-\frac{\|\mathbf{x}-\mathbf{y}\|_{2}^{2}}{2 \sigma^{2}}\right). \\
\text{Laplacian kernel:} & \quad k\left(\mathbf{x}, \mathbf{y}\right) = \exp \left(-\frac{\|\mathbf{x}-\mathbf{y}\|_{1}}{\sigma}\right).
\end{aligned}\nonumber
\end{equation}
Here \(c\), \(d\) and \(\sigma\) are the parameters of the kernels. For the parameter settings of the kernel functions, the parameters of the polynomial kernel are fixed as \( c=1 \) and \( d=\{2,4, 6\} \), and those of the Gaussian kernel and the Laplacian kernel are set as \( \sigma=\{10^t, t = -2, -1, 0, 1, 2\} \). For single kernel methods, we run KMFFS and KAUFS on each kernel separately. And we report both the best and the average results over all these kernels. For the MKAUFS method, which is a multiple kernel method, we implement it on a combination of the kernels mentioned above. It's worth noting that for the MKAUFS method, to avoid the influence of differences in magnitudes between different kernel matrices during optimization process, we need to standardize the centralized kernel. We use the following formula for all kernels:
\begin{equation}
\tilde{\mathbf{K}}_{i,j}=\frac{\mathbf{K}_{i,j}}{\sqrt{\mathbf{K}_{i,i}\mathbf{K}_{j,j}}}, \nonumber
\end{equation}
where \(i\) and \(j\) respectively represent the \(i\)-th row and \(j\)-th column of the kernel matrix \(\mathbf{K}\).

In the experiments, we utilize the k-means algorithm to cluster the samples based on the selected features, evaluating the UFS methods' performance through their clustering results. For simplicity, we set the number of clusters as the number of classes. It's also important to note that the k-means performance highly depends on the initial point selection. To mitigate this issue, we repeat the process 30 times with random initialization and report the average value and the standard deviation value of the performance measures ACC and NMI. For the redundancy measure, for each UFS method, we report the average redundancy rate of all feature subsets, which corresponds to the feature subset results where each achieves the best clustering performance.
\subsection{Result and Analysis}
We first present the numerical experimental results of all algorithms on eight datasets, and finally, we conduct parameter sensitivity tests on the KAUFS method and the MKAUFS method.

\subsubsection{Experimental Results}
To compare different UFS methods, we summarize their best clustering results in Tables \ref{acctable} and \ref{nmitable}. Particularly, for the KMFFS method and the KAUFS method, “-b” denotes the best result and “-a” means the average of those 14 kernels. In Table \ref{redtable}, we report the average redundancy rate of these methods in all datasets. In Tables \ref{acctable}, \ref{nmitable}, and \ref{redtable}, we bold the best and underlined the second-best performance results.

For the clustering performance, the ACC results given in Table \ref{acctable} and NMI results given in Table \ref{nmitable}, exhibit that our proposed methods provide the best results in most cases. Moreover, among the other state-of-the-art methods, DRFSMFMR method and KMFFS performs the best in most datasets. The difference between DRFSMFMR and our methods demonstrates the advantage of the kernel approach, while the contrast between KMFFS and our methods highlights the advantage of kernel alignment as a subspace distance. However, for the ACC results, it's worth noting that our proposed methods, along with others such as DRFSMFMR and LS-CAE, exhibit nearly equivalent efficacy to the Baseline on dataset Yale64. Particularly, in this dataset, the SCFS method outperforms the other methods. 

In addition, in the case of single kernel methods like KMFFS and KAUFS, significant differences between the best and average results are noticeable across datasets Yale64, Orlraws, and GLIOMA. This observation aligns with the general conclusion that the performance of single kernel methods is typically contingent upon the choice of kernel function. Such findings also motivate the development of multiple kernel learning techniques. Furthermore, we can observe that the MKAUFS method frequently achieves close or even superior performance compared to the results obtained using the best single kernel of the KAUFS method. This observation reveals the practical utility of multiple kernel learning methods, especially considering the avoidance of exhaustive searches across numerous candidate kernels.

For the average redundancy rate results in Table \ref{redtable}, we observe that the feature subsets selected by the three kernel-based algorithms, namely KMFFS, KAUFS, and MKAUFS methods, generally exhibit relatively low redundancy rates. This suggests the effectiveness of kernel methods in capturing nonlinear relationships between features, thus enhancing the independence among the selected features. It is evident that high redundancy rates are observed for DRFSMFMR and VCSDFS, which solely consider linear relationships and cannot capture more complex relationships among features, resulting in higher redundancy rates. Ranking second are the Laplacian-based LS and LS-CAE methods, whose performance also reflects the effectiveness of Laplacian-based feature selection algorithms in characterizing the manifold structure of the data, thereby better capturing both linear and nonlinear relationships between features. However, as these methods evaluate each feature individually, the interaction between features is ignored, which could potentially compromise clustering performance.

\begin{table}[H]\centering
\caption {Clustering results \(\left(ACC\%\pm std\%\right)\) of different feature selection algorithms on different datasets. The best results are bolded and the second best results are underlined (the higher the better).} \label{acctable}
\vspace{5pt}
\begin{tabu} to \textwidth { X[l,m] X[l,m] X[l,m] X[l,m] X[l,m]}
\hline
Dataset & Baseline & LS & LS-CAE  & VCSDFS \\
\hline
WarpAR & 25.89 \(\pm\) 3.28 & 27.53 \(\pm\) 2.54 & 38.28 \(\pm\) 3.44 & 41.27 \(\pm\) 1.58  \\
WarpPIE & 26.19 \(\pm\) 1.52 & 35.44 \(\pm\) 1.28 & 42.09 \(\pm\) 3.78 & 40.88 \(\pm\) 4.32 \\
Yale64 & 48.63 \(\pm\) 3.55 & 41.35 \(\pm\) 2.64 & 48.52 \(\pm\) 2.34 & 46.62 \(\pm\) 2.59\\
Orlraws & 75.77 \(\pm\) 4.91 & 76.10 \(\pm\) 1.71 & 75.36 \(\pm\) 3.81 & 69.86 \(\pm\) 5.11\\
Madelon & 50.32 \(\pm\) 0.41 & 54.67 \(\pm\) 1.47 & 55.22 \(\pm\) 0.05 & 58.63 \(\pm\) 0.13\\
GLIOMA & 56.53 \(\pm\) 3.99 & 52.37 \(\pm\) 3.86 & 52.95 \(\pm\) 0.01 & 62.33 \(\pm\) 4.95\\
TOX\_171 & 41.85 \(\pm\) 2.31 & 40.48 \(\pm\) 2.16 & 51.28 \(\pm\) 2.94 & 49.02 \(\pm\) 3.51\\
Prostate\_GE & 58.62 \(\pm\) 0.00 & 60.68 \(\pm\) 0.00 & 69.10 \(\pm\) 0.40 & 62.74 \(\pm\) 1.34\\
\hline
\end{tabu}
\begin{tabu} to \textwidth { X[l,m] X[l,m] X[l,m] X[l,m] X[l,m]}
Dataset & DRFSMFMR & SCFS &  KMFFS-b & KMFFS-a\\
\hline
WarpAR & 44.28 \(\pm\) 2.24 & 34.23 \(\pm\) 2.15 & 40.40 \(\pm\) 3.37 & 36.93\\
WarpPIE & 43.57 \(\pm\) 3.71 & 35.33 \(\pm\) 1.75 & 42.76 \(\pm\) 3.42 & 33.85\\
Yale64 & 48.64 \(\pm\) 4.34 & \textbf{59.28} \(\pm\) \textbf{2.55} & \underline{54.02 \(\pm\) 1.85} & 46.31\\
Orlraws & 74.36 \(\pm\) 3.42 & 70.63 \(\pm\) 3.29 & 78.90 \(\pm\) 3.88 & 71.66\\
Madelon & 62.65 \(\pm\) 0.16 & 59.81 \(\pm\) 1.08 & 61.18 \(\pm\) 1.56 & 59.27\\
GLIOMA & 58.23 \(\pm\) 4.76 & 57.78 \(\pm\) 0.81 & 62.26 \(\pm\) 8.02 & 45.32\\
TOX\_171 & 58.50 \(\pm\) 2.64 & 55.56 \(\pm\) 0.06 & 46.61 \(\pm\) 4.99 & 42.73\\
Prostate\_GE & 68.73 \(\pm\) 4.51 & 59.81 \(\pm\) 0.48 & 62.41 \(\pm\) 2.33 & 52.78\\
\hline
\end{tabu}
\begin{tabu} to \textwidth { X[l,m] X[l,m] X[l,m] X[l,m]}
Dataset &  KAUFS-b & KAUFS-a & MKAUFS\\
\hline
WarpAR & \underline{46.22 \(\pm\) 2.18} & 40.38 & \textbf{49.23 \(\pm\) 2.66}\\
WarpPIE & \underline{43.94 \(\pm\) 1.29} & 38.63 & \textbf{45.57 \(\pm\) 2.50} \\
Yale64 & 49.17 \(\pm\) 3.48 & 35.29 & 49.82 \(\pm\) 4.18 \\
Orlraws & \underline{88.33 \(\pm\) 6.69} & 71.16 & \textbf{89.37 \(\pm\) 5.41} \\
Madelon & \underline{66.98 \(\pm\) 1.95} & 64.46 & \textbf{67.04 \(\pm\) 1.13}\\
GLIOMA & \underline{67.13 \(\pm\) 4.81} & 53.78 & \textbf{68.24 \(\pm\) 0.73} \\
TOX\_171 & \underline{60.47 \(\pm\) 3.15} & 54.27 & \textbf{60.88 \(\pm\) 3.77} \\
Prostate\_GE & \underline{71.51 \(\pm\) 4.45} & 65.68 & \textbf{73.24 \(\pm\) 1.26}\\
\hline
\end{tabu}
\end{table}

\begin{table}[H]\centering
\caption {Clustering results \(\left(NMI\%\pm std\%\right)\) of different feature selection algorithms on different datasets. The best results are bolded and the second best results are underlined (the higher the better).} \label{nmitable}
\vspace{5pt}
\begin{tabu} to \textwidth { X[l,m] X[l,m] X[l,m] X[l,m] X[l,m]}
\hline
Dataset & Baseline & LS & LS-CAE  & VCSDFS \\
\hline
WarpAR & 28.29 \(\pm\) 2.44 & 27.45 \(\pm\) 2.52 & 42.48 \(\pm\) 2.54 & 48.72 \(\pm\) 6.95  \\
WarpPIE & 26.72 \(\pm\) 2.07 & 33.29 \(\pm\) 1.60 & 44.87 \(\pm\) 2.02 & \underline{48.04 \(\pm\) 1.98} \\
Yale64 & 56.42 \(\pm\) 2.27 & 51.93 \(\pm\) 2.45 & 54.77 \(\pm\) 2.32 & 53.30 \(\pm\) 2.38\\
Orlraws & 80.66 \(\pm\) 3.97 & 75.43 \(\pm\) 1.47 & 81.54 \(\pm\) 2.51 & 77.41 \(\pm\) 3.12\\
Madelon & 2.54 \(\pm\) 0.00 & 7.88 \(\pm\) 0.45 & 8.49 \(\pm\) 0.05 & 7.46 \(\pm\) 0.00\\
GLIOMA & 49.85 \(\pm\) 2.68 & 48.28 \(\pm\) 2.38 & 46.32 \(\pm\) 5.18 & 50.54 \(\pm\) 3.55\\
TOX\_171 & 14.23 \(\pm\) 2.09 & 19.91 \(\pm\) 2.57 & 26.84 \(\pm\) 1.78 & \textbf{37.21} \(\pm\) \textbf{4.32}\\
Prostate\_GE & 2.62 \(\pm\) 0.00 & 6.49 \(\pm\) 0.00 & 7.78 \(\pm\) 0.45 & 7.44 \(\pm\) 0.67\\
\hline
\end{tabu}
\begin{tabu} to \textwidth { X[l,m] X[l,m] X[l,m] X[l,m] X[l,m]}
Dataset & DRFSMFMR & SCFS &  KMFFS-b & KMFFS-a\\
\hline
WarpAR & 42.07 \(\pm\) 2.30 & 31.37 \(\pm\) 1.61 & 40.86 \(\pm\) 2.08 & 32.49\\
WarpPIE & 47.32 \(\pm\) 2.77 & 41.61 \(\pm\) 1.46 & 44.74 \(\pm\) 3.23 & 35.21\\
Yale64 & 55.04 \(\pm\) 2.11 & \textbf{64.28} \(\pm\) \textbf{1.84} & \underline{62.24 \(\pm\) 2.36} & 55.04\\
Orlraws & 78.11 \(\pm\) 2.31 & 79.83 \(\pm\) 2.58 & 84.87 \(\pm\) 3.11 & 78.45\\
Madelon & 12.13 \(\pm\) 1.10 & 3.46 \(\pm\) 0.00 & 10.65 \(\pm\) 1.69 & 2.25\\
GLIOMA & 53.98 \(\pm\) 4.35 & 51.61 \(\pm\) 3.03 & 53.06 \(\pm\) 6.75 & 25.62\\
TOX\_171 & 33.58 \(\pm\) 1.75 & 34.62 \(\pm\) 0.79 & \underline{35.64 \(\pm\) 5.37} & 24.32\\
Prostate\_GE & \textbf{18.13} \(\pm\) \textbf{1.83} & 7.46 \(\pm\) 1.24 & 8.90 \(\pm\) 3.67 & 3.79\\
\hline
\end{tabu}
\begin{tabu} to \textwidth { X[l,m] X[l,m] X[l,m] X[l,m]}
Dataset &  KAUFS-b & KAUFS-a & MKAUFS\\
\hline
WarpAR & \underline{50.71 \(\pm\) 4.37} & 41.71 & \textbf{54.04 \(\pm\) 1.92}\\
WarpPIE & 45.81 \(\pm\) 3.44 & 39.18 & \textbf{48.82 \(\pm\) 1.84} \\
Yale64 & 50.79 \(\pm\) 1.93 & 41.58 & 52.10 \(\pm\) 2.56 \\
Orlraws & \underline{86.67 \(\pm\) 0.24} & 74.33 & \textbf{87.21 \(\pm\) 2.32} \\
Madelon & \underline{15.36 \(\pm\) 0.78} & 9.23 & \textbf{16.01 \(\pm\) 2.19}\\
GLIOMA & \underline{59.11 \(\pm\) 3.83} & 44.68 & \textbf{59.33 \(\pm\) 3.07} \\
TOX\_171 & 29.47 \(\pm\) 2.84 & 24.22 & 32.21 \(\pm\) 4.04 \\
Prostate\_GE & 15.46 \(\pm\) 1.65 & 10.51 & \underline{17.97 \(\pm\) 1.92}\\
\hline
\end{tabu}
\end{table}

\begin{table}[H]\centering
\caption {Average redundancy rates (average RED\%) of different feature selection algorithms on different datasets. The best results are bolded and the second best results are underlined (the lower the better).} \label{redtable}
\vspace{5pt}
\begin{tabu} to \textwidth { X[l,m] X[l,m] X[l,m] X[l,m] X[l,m]}
\hline
Dataset & LS & LS-CAE  & VCSDFS & DRFSMFMR\\
\hline
WarpAR & 49.29 & 45.39 & 62.97 & 52.91 \\
WarpPIE & 67.66 & 47.58 & 71.85 & 70.19 \\
Yale64 & 34.64 & 32.79 & 42.25 & 40.76\\
Orlraws & 39.55 & 35.06 & 41.31 & 46.60\\
Madelon & 3.58 & 3.67 & 4.32 & 4.12\\
GLIOMA & \underline{38.10} & 38.21 & 48.31 & 44.47\\
TOX\_171 & 18.20 & 19.31 & 21.95 & 26.34\\
Prostate\_GE & 45.56 & 48.2 & 50.18 & 47.06\\
\hline
\end{tabu}
\begin{tabu} to \textwidth { X[l,m] X[l,m] X[l,m] X[l,m] X[l,m]}
Dataset & SCFS &  KMFFS & KAUFS & MKAUFS\\
\hline
WarpAR & 61.70 & \underline{39.63} & \textbf{38.44} & 42.89\\
WarpPIE & 41.78 & \underline{43.89} & \textbf{40.33} & 44.21\\
Yale64 & 28.63 & 29.51 & \textbf{23.51} & \underline{28.25}\\
Orlraws & 42.35 & 32.18 & \textbf{26.98} & \underline{27.59}\\
Madelon & 4.76 & \textbf{3.27} & 3.47 & \underline{3.32}\\
GLIOMA & \textbf{37.29} & 42.73 & 43.69 & 39.82\\
TOX\_171 & 27.87 & 17.16 & \underline{16.23} & \textbf{16.17}\\
Prostate\_GE & 66.97 & 38.67 & \underline{34.59} & \textbf{31.71}\\
\hline
\end{tabu}
\end{table}

\subsubsection{Parameter Analysis}

Our models KAUFS and MKAUFS have two main important trade-off parameters, \(\alpha\) and \(\beta\), which control inner-product regularizations to capture local correlation information among features. We empirically assess their impact by varying \(\alpha\) and \(\beta\) within the range \(\{10^t, t = -3, -2, -1, 0, 1, 2, 3, 4\}\), and setting \(\gamma\) to 1. We fix the number of selected features at 100 and compute the clustering performance (ACC and NMI) of these two algorithms. Figures \ref{fig1} and \ref{fig2} depict how the clustering results in terms of ACC and NMI evolve with these penalty parameters on the WarpAR dataset and Yale64 datasets respectively. The figures reveal that the performance of KAUFS and MKAUFS exhibits stability across a broad range of \(\alpha\) and \(\beta\) values.

\begin{figure}[H]\centering
\includegraphics[width=1\textwidth]{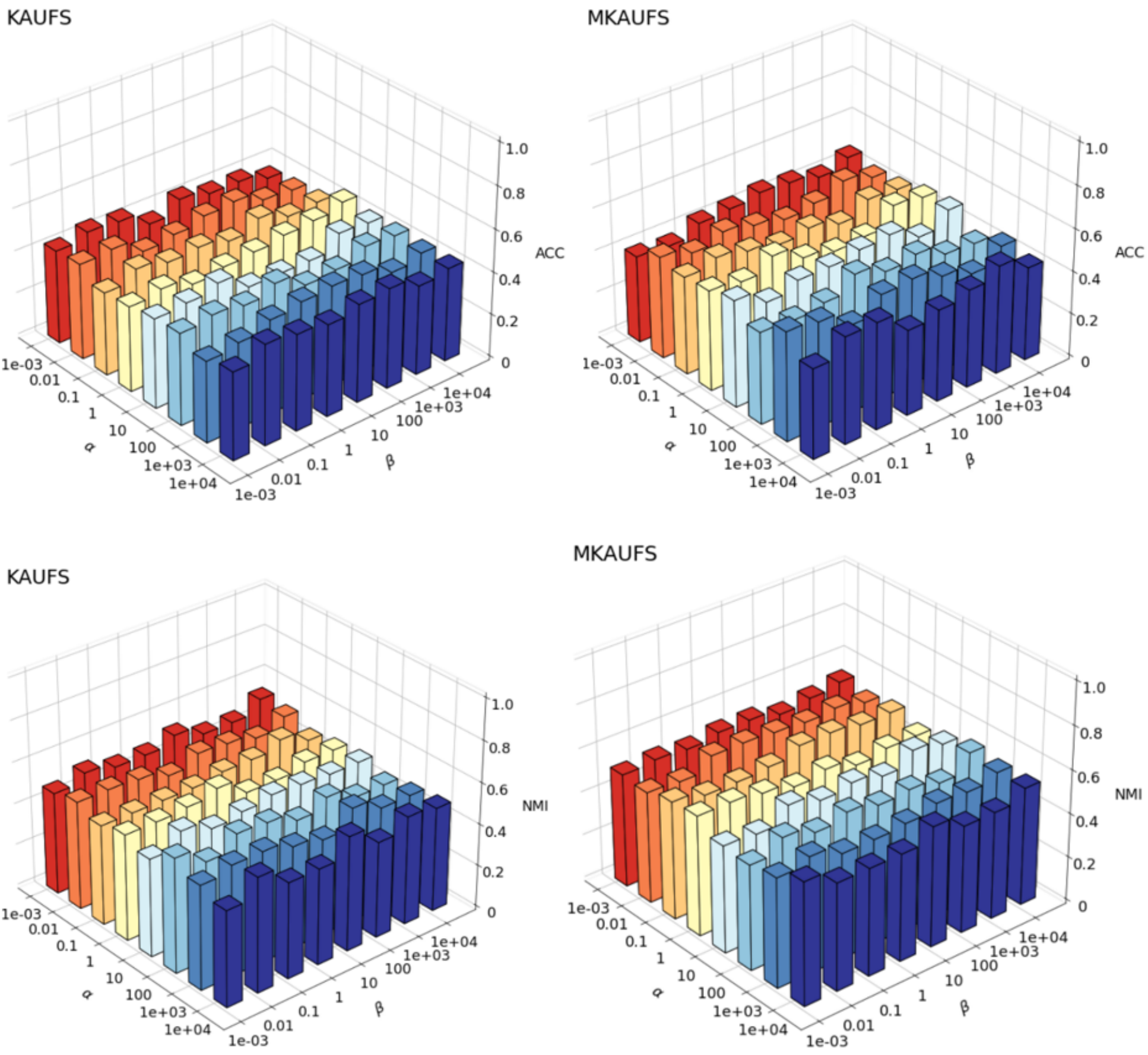}
\caption{Clustering performance ACC and NMI of KAUFS and MKAUFS w.r.t \(\alpha\) and \(\beta\) on the WarpAR dataset.}\label{fig1}
\end{figure}

\begin{figure}[H]\centering
\includegraphics[width=1\textwidth]{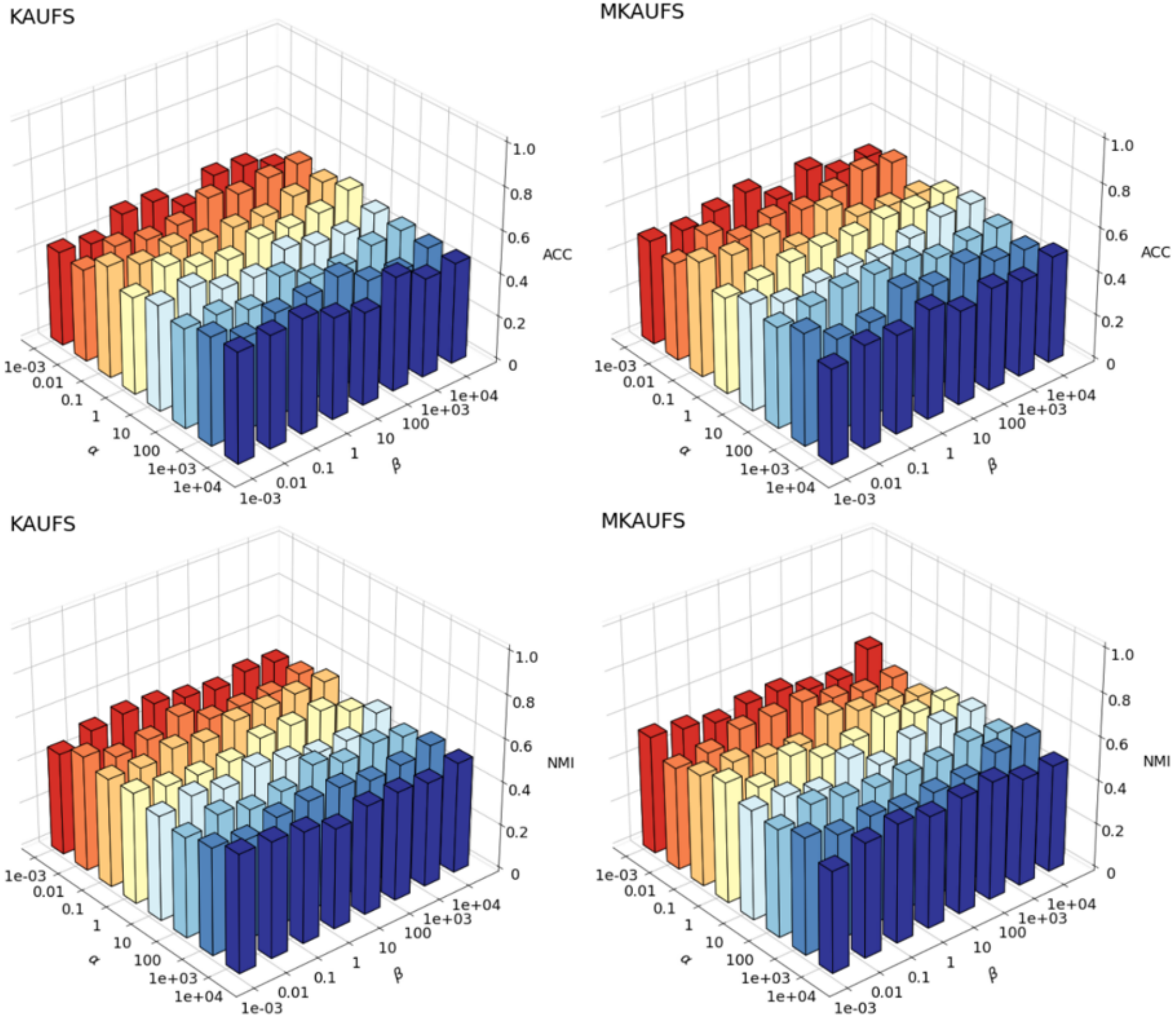}
\caption{Clustering performance ACC and NMI of KAUFS and MKAUFS w.r.t \(\alpha\) and \(\beta\) on the Yale64 dataset.}\label{fig2}
\end{figure}

\section{Conclusions}\label{Conclusions}

In this paper, we propose an unsupervised feature selection algorithm based on subspace learning and matrix factorization. We introduce a novel subspace distance, termed as the kernel alignment subspace distance, capable of capturing the nonlinear information embedded in the dataset features. Additionally, to further reduce the redundancy of the selected features subset, two inner product-based regularizations were employed to capture the dependence information among the selected features during the feature selection process. Furthermore, an efficient algorithm was developed to address the optimization problem of the proposed model, and we extend the model to a multi-kernel framework, which effectively learns an appropriate consensus kernel from a pool of kernels. This approach reduces the complexity of kernel selection and therefore holds greater practical significance. Finally, our extensive experiments on real datasets clearly demonstrated that our methods can effectively identify the most representative features with minimal redundancy. There are several interesting future directions. In particular, 1) improving subspace distance measurement methods could allow for the use of non-positive definite kernels, 2) utilizing the manifold structure information of the dataset, such as constructing a Laplacian graph and incorporating this information into the regularization framework, can also serve as a promising tool for enhancing the clustering performance and reducing redundancy of the feature subset, and 3) developing fast and reliable optimization algorithms to search for the global optimum of such optimization problems, rather than using matrix-based gradient descent algorithms to find local minima.

\clearpage
\bibliographystyle{unsrt}

\begin{thebibliography}{10}

\bibitem{parvaiz2023vision}
Arshi Parvaiz, Muhammad~Anwaar Khalid, Rukhsana Zafar, Huma Ameer, Muhammad Ali, and Muhammad~Moazam Fraz.
\newblock Vision transformers in medical computer vision—a contemplative retrospection.
\newblock {\em Engineering Applications of Artificial Intelligence}, 122:106126, 2023.

\bibitem{roy2023wildect}
Arunabha~M Roy, Jayabrata Bhaduri, Teerath Kumar, and Kislay Raj.
\newblock Wildect-yolo: An efficient and robust computer vision-based accurate object localization model for automated endangered wildlife detection.
\newblock {\em Ecological Informatics}, 75:101919, 2023.

\bibitem{chen2023graph}
Can Chen, Scott~T Weiss, and Yang-Yu Liu.
\newblock Graph convolutional network-based feature selection for high-dimensional and low-sample size data.
\newblock {\em Bioinformatics}, 39(4):btad135, 2023.

\bibitem{sharma2023hybrid}
Sanur Sharma and Anurag Jain.
\newblock Hybrid ensemble learning with feature selection for sentiment classification in social media.
\newblock In {\em Research Anthology on Applying Social Networking Strategies to Classrooms and Libraries}, pages 1183--1203. IGI Global, 2023.

\bibitem{parsons2004subspace}
Lance Parsons, Ehtesham Haque, and Huan Liu.
\newblock Subspace clustering for high dimensional data: a review.
\newblock {\em Acm sigkdd explorations newsletter}, 6(1):90--105, 2004.

\bibitem{fan2008high}
Jianqing Fan and Yingying Fan.
\newblock High dimensional classification using features annealed independence rules.
\newblock {\em Annals of statistics}, 36(6):2605, 2008.

\bibitem{guyon2003introduction}
Isabelle Guyon and Andr{\'e} Elisseeff.
\newblock An introduction to variable and feature selection.
\newblock {\em Journal of machine learning research}, 3(Mar):1157--1182, 2003.

\bibitem{guyon2008feature}
Isabelle Guyon, Steve Gunn, Masoud Nikravesh, and Lofti~A Zadeh.
\newblock {\em Feature extraction: foundations and applications}, volume 207.
\newblock Springer, 2008.

\bibitem{solorio2020review}
Sa{\'u}l Solorio-Fern{\'a}ndez, J~Ariel Carrasco-Ochoa, and Jos{\'e}~Fco Mart{\'\i}nez-Trinidad.
\newblock A review of unsupervised feature selection methods.
\newblock {\em Artificial Intelligence Review}, 53(2):907--948, 2020.

\bibitem{zhou2016global}
Nan Zhou, Yangyang Xu, Hong Cheng, Jun Fang, and Witold Pedrycz.
\newblock Global and local structure preserving sparse subspace learning: An iterative approach to unsupervised feature selection.
\newblock {\em Pattern Recognition}, 53:87--101, 2016.

\bibitem{wang2020discriminative}
Zheng Wang, Feiping Nie, Lai Tian, Rong Wang, and Xuelong Li.
\newblock Discriminative feature selection via a structured sparse subspace learning module.
\newblock In {\em IJCAI}, pages 3009--3015, 2020.

\bibitem{parsa2020unsupervised}
Mohsen~Ghassemi Parsa, Hadi Zare, and Mehdi Ghatee.
\newblock Unsupervised feature selection based on adaptive similarity learning and subspace clustering.
\newblock {\em Engineering Applications of Artificial Intelligence}, 95:103855, 2020.

\bibitem{nie2020subspace}
Feiping Nie, Zheng Wang, Lai Tian, Rong Wang, and Xuelong Li.
\newblock Subspace sparse discriminative feature selection.
\newblock {\em IEEE Transactions on Cybernetics}, 52(6):4221--4233, 2020.

\bibitem{wang2015subspace}
Shiping Wang, Witold Pedrycz, Qingxin Zhu, and William Zhu.
\newblock Subspace learning for unsupervised feature selection via matrix factorization.
\newblock {\em Pattern Recognition}, 48(1):10--19, 2015.

\bibitem{zhu2015unsupervised}
Pengfei Zhu, Wangmeng Zuo, Lei Zhang, Qinghua Hu, and Simon~CK Shiu.
\newblock Unsupervised feature selection by regularized self-representation.
\newblock {\em Pattern Recognition}, 48(2):438--446, 2015.

\bibitem{karami2023unsupervised}
Saeed Karami, Farid Saberi-Movahed, Prayag Tiwari, Pekka Marttinen, and Sahar Vahdati.
\newblock Unsupervised feature selection based on variance-covariance subspace distance.
\newblock {\em Neural Networks}, 2023.

\bibitem{wang2015unsupervised}
Shiping Wang, Witold Pedrycz, Qingxin Zhu, and William Zhu.
\newblock Unsupervised feature selection via maximum projection and minimum redundancy.
\newblock {\em Knowledge-Based Systems}, 75:19--29, 2015.

\bibitem{qi2018unsupervised}
Miao Qi, Ting Wang, Fucong Liu, Baoxue Zhang, Jianzhong Wang, and Yugen Yi.
\newblock Unsupervised feature selection by regularized matrix factorization.
\newblock {\em Neurocomputing}, 273:593--610, 2018.

\bibitem{cristianini2001kernel}
Nello Cristianini, John Shawe-Taylor, Andre Elisseeff, and Jaz Kandola.
\newblock On kernel-target alignment.
\newblock {\em Advances in neural information processing systems}, 14, 2001.

\bibitem{cortes2012algorithms}
Corinna Cortes, Mehryar Mohri, and Afshin Rostamizadeh.
\newblock Algorithms for learning kernels based on centered alignment.
\newblock {\em The Journal of Machine Learning Research}, 13(1):795--828, 2012.

\bibitem{wei2016nonlinear}
Xiaokai Wei, Bokai Cao, and Philip~S Yu.
\newblock Nonlinear joint unsupervised feature selection.
\newblock In {\em Proceedings of the 2016 SIAM International Conference on Data Mining}, pages 414--422. SIAM, 2016.

\bibitem{xing2021fairness}
Xiaoying Xing, Hongfu Liu, Chen Chen, and Jundong Li.
\newblock Fairness-aware unsupervised feature selection.
\newblock In {\em Proceedings of the 30th ACM International Conference on Information \& Knowledge Management}, pages 3548--3552, 2021.

\bibitem{palazzo2020unsupervised}
Martin Palazzo, Pierre Beauseroy, and Patricio Yankilevich.
\newblock Unsupervised feature selection for tumor profiles using autoencoders and kernel methods.
\newblock In {\em 2020 IEEE Conference on Computational Intelligence in Bioinformatics and Computational Biology (CIBCB)}, pages 1--8. IEEE, 2020.

\bibitem{gonen2011multiple}
Mehmet G{\"o}nen and Ethem Alpayd{\i}n.
\newblock Multiple kernel learning algorithms.
\newblock {\em The Journal of Machine Learning Research}, 12:2211--2268, 2011.

\bibitem{shang2019local}
Ronghua Shang, Yang Meng, Wenbing Wang, Fanhua Shang, and Licheng Jiao.
\newblock Local discriminative based sparse subspace learning for feature selection.
\newblock {\em Pattern Recognition}, 92:219--230, 2019.

\bibitem{li2022unsupervised}
Weiyi Li, Hongmei Chen, Tianrui Li, Jihong Wan, and Binbin Sang.
\newblock Unsupervised feature selection via self-paced learning and low-redundant regularization.
\newblock {\em Knowledge-Based Systems}, 240:108150, 2022.

\bibitem{nie2010efficient}
Feiping Nie, Heng Huang, Xiao Cai, and Chris Ding.
\newblock Efficient and robust feature selection via joint l2, 1-norms minimization.
\newblock {\em Advances in neural information processing systems}, 23, 2010.

\bibitem{saberi2022dual}
Farid Saberi-Movahed, Mehrdad Rostami, Kamal Berahmand, Saeed Karami, Prayag Tiwari, Mourad Oussalah, and Shahab~S Band.
\newblock Dual regularized unsupervised feature selection based on matrix factorization and minimum redundancy with application in gene selection.
\newblock {\em Knowledge-Based Systems}, 256:109884, 2022.

\bibitem{han2015selecting}
Jiuqi Han, Zhengya Sun, and Hongwei Hao.
\newblock Selecting feature subset with sparsity and low redundancy for unsupervised learning.
\newblock {\em Knowledge-Based Systems}, 86:210--223, 2015.

\bibitem{ding2008convex}
Chris~HQ Ding, Tao Li, and Michael~I Jordan.
\newblock Convex and semi-nonnegative matrix factorizations.
\newblock {\em IEEE transactions on pattern analysis and machine intelligence}, 32(1):45--55, 2008.

\bibitem{lee2000algorithms}
Daniel Lee and H~Sebastian Seung.
\newblock Algorithms for non-negative matrix factorization.
\newblock {\em Advances in neural information processing systems}, 13, 2000.

\bibitem{ding2006orthogonal}
Chris Ding, Tao Li, Wei Peng, and Haesun Park.
\newblock Orthogonal nonnegative matrix t-factorizations for clustering.
\newblock In {\em Proceedings of the 12th ACM SIGKDD international conference on Knowledge discovery and data mining}, pages 126--135, 2006.

\bibitem{pavlidis2001gene}
Paul Pavlidis, Jason Weston, Jinsong Cai, and William~Noble Grundy.
\newblock Gene functional classification from heterogeneous data.
\newblock In {\em Proceedings of the fifth annual international conference on Computational biology}, pages 249--255, 2001.

\bibitem{mariette2018unsupervised}
J{\'e}r{\^o}me Mariette and Nathalie Villa-Vialaneix.
\newblock Unsupervised multiple kernel learning for heterogeneous data integration.
\newblock {\em Bioinformatics}, 34(6):1009--1015, 2018.

\bibitem{sonnenburg2006large}
S{\"o}ren Sonnenburg, Gunnar R{\"a}tsch, Christin Sch{\"a}fer, and Bernhard Sch{\"o}lkopf.
\newblock Large scale multiple kernel learning.
\newblock {\em The Journal of Machine Learning Research}, 7:1531--1565, 2006.

\bibitem{rakotomamonjy2007more}
Alain Rakotomamonjy, Francis Bach, St{\'e}phane Canu, and Yves Grandvalet.
\newblock More efficiency in multiple kernel learning.
\newblock In {\em Proceedings of the 24th international conference on Machine learning}, pages 775--782, 2007.

\bibitem{zheng2006non}
Danian Zheng, Jiaxin Wang, and Yannan Zhao.
\newblock Non-flat function estimation with a multi-scale support vector regression.
\newblock {\em Neurocomputing}, 70(1-3):420--429, 2006.

\bibitem{kang2017kernel}
Zhao Kang, Chong Peng, and Qiang Cheng.
\newblock Kernel-driven similarity learning.
\newblock {\em Neurocomputing}, 267:210--219, 2017.

\bibitem{li2017feature}
Jundong Li, Kewei Cheng, Suhang Wang, Fred Morstatter, Robert~P Trevino, Jiliang Tang, and Huan Liu.
\newblock Feature selection: A data perspective.
\newblock {\em ACM computing surveys (CSUR)}, 50(6):1--45, 2017.

\bibitem{he2005laplacian}
Xiaofei He, Deng Cai, and Partha Niyogi.
\newblock Laplacian score for feature selection.
\newblock {\em Advances in neural information processing systems}, 18, 2005.

\bibitem{shaham2022deep}
Uri Shaham, Ofir Lindenbaum, Jonathan Svirsky, and Yuval Kluger.
\newblock Deep unsupervised feature selection by discarding nuisance and correlated features.
\newblock {\em Neural Networks}, 152:34--43, 2022.

\bibitem{szekely2007measuring}
G{\'a}bor~J Sz{\'e}kely, Maria~L Rizzo, and Nail~K Bakirov.
\newblock Measuring and testing dependence by correlation of distances.
\newblock 2007.

\end{thebibliography}

\end{document}